\documentclass[options]{article} 
\usepackage{amssymb}
\usepackage{amsfonts}
\usepackage{amsmath}
\usepackage{amsthm}
\usepackage{algorithmic}
\usepackage{algorithm}
\usepackage{booktabs}
\usepackage{url}
\usepackage{multirow}
\usepackage{IEEEtrantools}
\usepackage{graphicx,color}
\usepackage[bf,SL,BF]{subfigure}
\usepackage{natbib}
\usepackage{color,xcolor}
\usepackage{soul} 
\usepackage[colorlinks, citecolor=blue, linkcolor=black]{hyperref}
\usepackage{enumerate}
\usepackage[title]{appendix}
\usepackage{lscape}
\usepackage{makecell}
\usepackage{colortbl}
\newcommand{\bPsi}{\boldsymbol{\Psi}}\newcommand{\bS}{\mathbf{S}}\newcommand{\bB}{\mathbf{B}}\newcommand{\bM}{\mathbf{M}}
\newcommand{\bC}{\mathbf{C}}\newcommand{\bX}{\mathbf{X}}\newcommand{\bP}{\mathbf{P}}\newcommand{\bR}{\mathbf{R}}\newcommand{\bY}{\mathbf{Y}}\newcommand{\bU}{\mathbf{U}}\newcommand{\bA}{\mathbf{A}}\newcommand{\bI}{\mathbf{I}}\newcommand{\bL}{\mathbf{L}}\newcommand{\bW}{\mathbf{W}}\newcommand{\bH}{\mathbf{H}}\newcommand{\bLmd}{\boldsymbol{\Lambda}}\newcommand{\bV}{\mathbf{V}}\newcommand{\bZ}{\mathbf{Z}}\newcommand{\bo}{\mathbf{0}}\newcommand{\be}{\mathbf{e}}\newcommand{\bx}{\mathbf{x}}\newcommand{\bz}{\mathbf{z}}\newcommand{\bu}{\mathbf{u}} \newcommand{\bs}{\mathbf{s}}\newcommand{\bb}{\mathbf{b}}\def\b1{{\mathbf1}}
\newcommand{\bO}{\mathbf{O}}

\newcommand{\bbE}{\mathbb{E}}

\newcommand{\bTG}{\mathrm{TG}}
\usepackage{xspace}
\newcommand{\et}{\emph{t}\xspace}
\newcommand{\tfa}{\emph{t}FA\xspace}
\newcommand{\tbfa}{\emph{t}BFA\xspace}
\newcommand{\mt}{\emph{Mt}\xspace}

\newcommand{\btS}{\widetilde{\mathbf{S}}}\newcommand{\btC}{\widetilde{\mathbf{C}}}\newcommand{\btR}{\widetilde{\mathbf{R}}}\newcommand{\btSig}{\widetilde{\boldsymbol{\Sigma}}}\newcommand{\btW}{\widetilde{\mathbf{W}}}\newcommand{\btPsi}{\widetilde{\bPsi}}

\newcommand{\cD}{\mathcal{D}}\newcommand{\cL}{\mathcal{L}}\newcommand{\cG}{\mathcal{G}}\newcommand{\cT}{\mathcal{T}}\newcommand{\cN}{\mathcal{N}}\newcommand{\cX}{\mathcal{X}}\newcommand{\cY}{\mathcal{Y}}

\newcommand{\bwA}{\widehat{\bA}}\newcommand{\bwC}{\widehat{\bC}}\newcommand{\bwR}{\widehat{\bR}}\newcommand{\bwSig}{\widehat{\bSig}}
\newcommand{\bwTheta}{\widehat{\bTheta}}
\newcommand{\bwPsi}{\widehat{\bPsi}}

 \newcommand{\bttheta}{\tilde{\btheta}}
\newcommand{\tpsi}{\tilde{\psi}}

\newcommand{\beps}{\boldsymbol{\epsilon}}\newcommand{\bmu}{\boldsymbol{\mu}}\newcommand{\btheta}{\boldsymbol{\theta}}\newcommand{\bTheta}{\boldsymbol{\Theta}} 
\newcommand{\bSig}{\boldsymbol{\Sigma}}

\newcommand\refe[1]{(\ref{#1})}\newcommand\refap[1]{\ref{#1}}
\newcommand\reff[1]{Fig.~\ref{#1}}\newcommand\reft[1]{Table~\ref{#1}}\newcommand\refs[1]{Sec.~\ref{#1}} 
  
 \newcommand\refthm[1]{Theorem~\ref{#1}}\newcommand\refdef[1]{Definition~\ref{#1}}
\theoremstyle{plain} \newtheorem{defi}{Definition} 
\newtheorem{thm}{Theorem}

\def\R{{\mathbb R}}

\def\tr{\mbox{tr}}
\def\vc{\mbox{vec}}

\def\diag{\mbox{diag}}

\def\s2{\sigma^2}
\def\Gam{\mbox{Gam}}
\def\vec{\mbox{vec}}
\def\vech{\mbox{vech}}

\newcommand{\blind}{0}

\begin{document}
	
	\def\spacingset#1{\renewcommand{\baselinestretch}%
		{#1}\small\normalsize} \spacingset{1}

	\if0\blind
	{
		\title{\bf Robust bilinear factor analysis based on the matrix-variate \emph{t} distribution}
		\author{Xuan Ma, Jianhua Zhao\thanks{The corresponding author. Email: jhzhao.ynu@gmail.com}\hspace{.2cm}, Changchun Shang, Fen Jiang and Philip L.H. Yu \\
			School of Statistics and Mathematics, Yunnan University of Finance and Economics}
		\maketitle
	} \fi

	\if1\blind
	{
		\bigskip
		\bigskip
		\bigskip
		\begin{center}
			{\LARGE\bf Robust bilinear factor analysis based on the matrix-variate \emph{t} distribution}
		\end{center}
		\medskip
	} \fi
	
	\bigskip
	
	\begin{abstract}
		Factor Analysis based on multivariate \emph{t} distribution (\tfa) is a useful robust tool for extracting common factors on heavy-tailed or contaminated data. However, \tfa is only applicable to vector data. When \tfa is applied to matrix data, it is common to first vectorize the matrix observations. This introduces two challenges for \tfa: (i) the inherent matrix structure of the data is broken, and (ii) robustness may be lost, as vectorized matrix data typically results in a high data dimension, which could easily lead to the breakdown of \tfa. To address these issues, starting from the intrinsic matrix structure of matrix data, a novel robust factor analysis model, namely bilinear factor analysis built on the matrix-variate \emph{t} distribution (\tbfa), is proposed in this paper. The novelty is that it is capable to simultaneously extract common factors for both row and column variables of interest on heavy-tailed or contaminated matrix data. Two efficient algorithms for maximum likelihood estimation of \tbfa are developed. Closed-form expression for the Fisher information matrix to calculate the accuracy of parameter estimates are derived. Empirical studies are conducted to understand the proposed \tbfa model and compare with related competitors. The results demonstrate the superiority and practicality of \tbfa. Importantly, \tbfa exhibits a significantly higher breakdown point than \tfa, making it more suitable for matrix data. 
	\end{abstract}
	
	\spacingset{1.75} 
	
	\section{Introduction}\label{sec:intr} 
	Factor Analysis (FA) is a powerful multivariate analysis method that identifies latent common characteristics (or factors) within a given set of variables, and has been widely used in diverse scientific areas including but not limited to dimension reduction, time series prediction, psychology, social sciences, and economics. 
	However, FA is based on the Gaussian assumption and hence is vulnerable
	to outliers. To overcome this problem, the FA model modeling with multivariate \emph{t} (\tfa) is suggested and has been proven to be effective for robust estimation \citep{zhang2014robust,mclachlan2007extension}. However, FA and \tfa are only applicable to vector-valued or vector data, where observations are vectors. Nowadays, matrix-valued data or matrix data, where observations are matrices, has become increasingly common in our real life. For instance,
	
	\emph{Data 1}. In water research \citep{oyying-water}, observations consist
	of values of variables such as water temperature, water PH value,
	oxygen concentration, nitrogen concentration, etc., measured at 22
	monitoring stations. The data can naturally be represented as a matrix with columns being 24 variables and rows being 22 different
	stations. The sample size of this data is 3. 
	
	When FA and \tfa are applied to matrix data such as \emph{Data 1}, it is common practice to first vectorize the matrix observations. This leads to two problems: (i) vectorization destroys the natural matrix structure, merging row and column variables together, making it incapable of capturing common factors among row variables and column variables of interest. Take \emph{Data 1} as an example, FA and \tfa can not extract the interesting common `station' factors and `water quality' factors; and (ii) vectorization often results in a very high dimension, making FA and \tfa easily trapped into the curse of dimensionality. For example, the dimension of the vectorized \emph{Data 1}, $d=528$, is much greater than the sample size of $N=3$. This is the case of the so-called $d>N$ problem, on which FA and \tfa would fail due to the numerical issues \citep{zhao2008-efa}. \tfa suffers one more problem, namely (iii) robustness may be lost, as the breakdown point of \tfa is upper bounded by the inverse of the vectorized data dimension \citep{Dumbgen-bkd} and the higher the dimension, the lower the breakdown point that measures the robustness.  
	
	To address problem (i) and (ii), several matrix-based methods have emerged as promising alternatives to vector-based ones, including probabilistic and non-probabilistic methods. An example for non-probabilistic methods is the matrix factor models (MFM) proposed in \cite{wang2019factor}. Probabilistic methods include matrix-variate principal component analysis (PCA) \citep{james-mvfa}, factored PCA (FPCA), and bilinear probabilistic PCA (BPPCA) \citep{zhao2012-bppca}. Based on BPPCA, \cite{zhao2023-BFA} present a bilinear factor analysis (BFA) model. These methods share a common characteristic: they make use of bilinear transformations with number of parameters typically significantly smaller than those in linear transformations used by vector-based models. As a result, they effectively mitigate high-dimensional problems. 
	
	Despite the advantages of these matrix-based probabilistic methods, they all rely on the assumption of normal distribution and are thereby susceptible to outliers. Recently, in order to improve the robustness of FPCA, \cite{zhao2023-rfpca} propose a robust extension of FPCA (RFPCA) using the matrix-variate \et (\mt) distribution \citep{gupta2013elliptically}. Inspired by RFPCA, in this paper we propose a new robust extension of BFA for matrix data, which is built on the \mt distribution, to simultaneously solve the problems (i), (ii) and (iii) suffered by \emph{t}FA. Throughout this paper, the proposed robust BFA model is referred to as \tbfa. The main contributions of this paper are threefold. 
	
	
	(i) Due to the utilization of a bilinear transformation, the proposed \tbfa alleviates greatly the high-dimensional problems and is capable to simultaneously extract common factors for both row and column variables of interest even in the presence of matrix outliers. Moreover, the proposed \tbfa is generally compared favorably with the related methods.
	
	(ii) Two efficient algorithms for ML estimation of \tbfa model parameters are developed and analytical expression for the Fisher information matrix to calculate the accuracy of parameter estimates are also derived. The factor scores in \tbfa are represented as matrices, instead of vectors in \tfa, visual analysis to the matrix factor scores are provided.
	
	(iii) More importantly, while a rigorous theoretical result is not yet available, our empirical results show that \tbfa has a significantly higher breakdown point than its vector-based counterpart \tfa, which makes \tbfa applicable to matrix data. 
	
	
	The rest of this paper is organized as follows. In \refs{sec:rlt}, we briefly review some related works. In \refs{sec:tbfa.model}, we propose our \tbfa, and develop its maximum likelihood inference in \refs{sec:tbfa.infer}. In \refs{sec:expr}, we conduct a number of experiments on simulated and real-world data. In \refs{sec:conclusion}, we close the paper with some conclusions and discussions.  
	
	
	\section{Related works}\label{sec:rlt}
	In this section, we briefly review the multivariate-\emph{t} factor analysis (\tfa) \citep{zhang2014robust,mclachlan2007extension} and bilinear factor analysis (BFA) \citep{zhao2023-BFA} models. The following notations are used throughout this paper. $\bI$ and $\bo$ represent identity and zero matrices of suitable dimensions, respectively. The transpose of vector $\bx$ or matrix $\bX$ is denoted by $\bx'$ or $\bX'$, the matrix trace by $\tr(\cdot)$ and the Kronecker product by $\otimes$. The half-vectorization $\vech(\bX)$ is the vector containing unique sub-diagonal elements of matrix $\bX$.

	
	\subsection{Multivariate-t factor analysis (tFA) }\label{sec:tfa}
	If a $d$-dimensional random vector $\bx$ follows the multivariate \emph{t} distribution with center $\bmu\in \R^{d}$, covariance matrix $\bSig\in \R^{d\times d}$ and degrees of freedom $\nu>0$, denoted by $t_d(\bmu,\bSig,\nu)$, the probability density function (p.d.f.) of $\bx$ is 
	\begin{equation}
		\hskip-1em p(\bx)=\frac{|\bSig|^{-\frac{1}{2}} \Gamma(\frac{\nu+d}{2})}{(\pi \nu)^{\frac{d}{2}}\Gamma(\frac{\nu}{2})}
		\left(1+\frac{\delta_\bx(\btheta)}{\nu}\right)^{-\frac{\nu+d}{2}},\label{eqn:t.pdf}
	\end{equation}
	where $\delta_\bx(\btheta)=(\bx-\bmu)'\bSig^{-1}(\bx-\bmu)$ is the squared Mahalanobis distance between $\bx$ and $\bmu$, and $\nu$ can tune robustness by adaptively controlling the heaviness of tails (kurtosis). The mean $\bmu$ and the covariance matrix $\bSig\nu/(\nu-2)$ exist when $\nu>1$ and $\nu>2$, respectively \citep{liu-tdist}. 
	
	The multivariate-\et distribution can be described as a hierarchical representation, as shown in \refe{eqn:t.hrc} \citep{liu-tdist}. Specifically, we introduce a latent variable $\tau$ following the Gamma distribution $\Gam(\nu/2, \nu/2)$. Given $\tau$, the conditional distribution of $\bx$ is a multivariate normal distribution $\cN_d(\bmu,\bSig/\tau)$.
	\begin{equation}\label{eqn:t.hrc}
		\bx|\tau\sim\cN_d(\bmu,\bSig/\tau),\,\,	\tau\sim\cG(\nu/2,\nu/2).
	\end{equation}
	
	Suppose the $d$-dimensional data vector $\bx$ follows a $q$-factor multivariate-\et factor analysis (\tfa) model. Based on the hierarchical representation in \refe{eqn:t.hrc}, the latent variable model for \tfa can be expressed as
	\begin{equation}\label{eqn:tfa}
		\left\{\begin{array}{l}
			\bx = \bA\bz + \bmu +  \beps,\quad \mbox{with}\,\,\tau\sim \cG(\nu/2,\nu/2),\\ 
			\bz|\tau \sim \cN_q(\bo,\bI/\tau),\quad  \beps|\tau \sim \cN_d(\bo,\bPsi),
		\end{array}\right.
	\end{equation}
	where $\bz  \in \R^{q}$ is the latent factor vector, $\bA \in \R^{d \times q}$ is the factor loading matrix, $\bmu \in \R^{d}$ is the mean vector, $\bPsi = \diag(\psi_{1}, \psi_{2}, \dots, \psi_{d})$ is a positive diagonal matrix. Assuming $\bH \in \R^{q \times q}$ is an orthogonal matrix, it is easy to see that model \refe{eqn:tfa} remains invariant if we replace $\bA$ by $\bA\bH$. Therefore, the parameter $\bA$ can only be determined up to an orthogonal rotation, and the number of free parameters in \tfa model is $\cD = d(q + 2) - q(q - 1)/2 + 1$. In addition, to avoid over-parameterization, the maximum number of factors $q$ should satisfy the inequality $q_{\text{max}} \leqslant d + (1 - \sqrt{1 + 8d})/2$ \citep{zhao2014-fa-auto}. 
	

	
	
	\subsection{Bilinear factor analysis (BFA)}\label{sec:bfa}
	The latent variable model of BFA for matrix data introduced in \cite{zhao2023-BFA} is defined as
	\begin{equation}\label{eqn:bfa}
		\left\{\begin{array}{l}
			\bX=\bC \bZ \bR^{\prime}+\bW+\mathbf{C} \beps_r+\beps_c \bR^{\prime}+\beps, \\
			\bZ \sim \cN_{q_{c}, q_{r}}(\bo, \bI, \bI), \quad \beps_r \sim \cN_{q_{c}, d_{r}}(\bo, \bI, \bPsi_r ), \\
			\beps_c \sim \cN_{d_{c}, q_{r}}(\bo, \bPsi_c, \bI), \quad \beps \sim \cN_{d_{c}, d_{r}}(\bo, \bPsi_c, \bPsi_r),
		\end{array}\right.
	\end{equation}
	where $\cN_{q_{c}, q_{r}}$ denotes the $q_c \times q_r$-dimensional matrix-variate normal distribution. $\bW \in \R^{d_c \times d_r}$ is the mean matrix, and $\bC\in\R^{d_c\times q_c}$ and $\bR\in\R^{d_r\times q_r}$ are the column and row factor loading matrices, respectively. $\bZ \in \R^{q_c \times q_r}$ is the latent common factor matrix, and $\beps_c \in \R^{d_c \times q_r}$, $\beps_r \in \R^{q_c \times d_r}$, and $\beps \in \R^{d_c \times d_r}$ represent the column, row, and common error terms, respectively. The positive diagonal matrices $\bPsi_c=\diag(\psi_{c,1},\psi_{c,2},\dots,\psi_{c,d_c})$ and $\bPsi_r=\diag(\psi_{r,1},\psi_{r,2},\dots,\psi_{r,d_r})$. It is assumed that $\bZ$, $\beps_c$, $\beps_r$, and $\beps$ are independent of each other. Notably, BFA model in \refe{eqn:bfa} will degenerate to the classic FA when $d_c =1$ or $d_r=1$.
	
	Like \tfa, the factor loading matrices $\bC$ and $\bR$ in BFA can only be uniquely determined up to orthogonal rotation matrices. Consequently, the numbers of free parameters in $\bC$ and $\bR$ are $d_cq_c-q_c(q_c-1)/2$ and $d_rq_r-q_r(q_r-1)/2$, respectively, and the total number of free parameters in BFA is $\cD=d_c(q_c+1)-q_c(q_c-1)/2 +d_r(q_r+1)-q_r(q_r-1)/2 + d_cd_r$. To avoid over-parameterization, the maximum numbers of column and row factors should satisfy the inequalities $q_{c,max} \leqslant d_c + (1 - \sqrt{1 + 8d_c})/2$ and $q_{r,max} \leqslant d_r + (1 - \sqrt{1 + 8d_r})/2$, respectively. 
	

	\section{Robust BFA based on a matrix-variate \emph{t} distribution (\tbfa)}\label{sec:tbfa}
	In this section, based on the \emph{Mt} distribution \citep{gupta2013elliptically}, we propose a novel robust BFA model (\tbfa). In \refs{sec:matrix-t}, we briefly review the \emph{Mt} distribution, based on which, we propose our \tbfa in \refs{sec:tbfa.model}. We discuss the model identifiability issue associated with \tbfa in \refs{sec:tbfa.identi}.
	
	\subsection{Matrix-variate t (Mt) distribution}\label{sec:matrix-t}
	\begin{defi}\label{def.matrix-t}
		Let $\bX\in\R^{d_c\times d_r}$ be a random matrix. It is said that $\bX$ follows a Mt distribution with the mean matrix $\bW\in\R^{d_c\times d_r}$, column and row covariance matrices $\bSig_c\in\R^{d_c\times d_c}$ and $\bSig_r\in\R^{d_r\times d_r}$, and degrees of freedom $\nu$, denoted by $\bX\sim Mt_{d_c,d_r}(\bW, \bSig_c, \bSig_r, \nu)$, if $\vc(\bX)\sim t_{d_cd_r}(\vc(\bW), \bSig_r\otimes\bSig_c,\nu)$. The p.d.f. of $\bX$ is given by
		\begin{equation}\label{eqn:mvt.density}
			p(\bX)=\frac{|\bSig_c|^{-\frac{d_r}{2}}|\bSig_r|^{-\frac{d_c}{2}} \Gamma(\frac{\nu+d_cd_r}{2})}{(\pi \nu)^{\frac{d_cd_r}{2}}\Gamma(\frac{\nu}{2})}\left(1+\frac{1}{\nu}\tr\left\{\bSig_c^{-1}(\bX-\bW)\bSig_r^{-1}(\bX-\bW)'\right\}\right)^{-\frac{\nu+d_cd_r}{2}}.
		\end{equation}
	\end{defi}
	
	$\bX$ can be expressed as a latent variable model. Introduce a latent weight variable $\tau$, which follows a Gamma distribution $\Gam(\nu/2, \nu/2)$. Given $\tau$, $\bX$ follows the matrix-normal distribution $\cN_{d_c,d_r}(\bW,\bSig_c/\tau,\bSig_r)$, that is, 
	\begin{equation}\label{eqn:mvt.hrc}
		\bX|\tau\sim\cN_{d_c,d_r}(\bW,\bSig_c/\tau,\bSig_r),\,\, \tau\sim\Gam(\nu/2,\nu/2).
	\end{equation}
	
	According to \refdef{def.matrix-t}, the \emph{Mt} distribution is essentially a multivariate \emph{t} distribution with a separable covariance structure $\bSig=\bSig_r\otimes \bSig_c$. Consequently, the \emph{Mt} distribution naturally inherits key characteristics of the multivariate \emph{t} distribution. Specifically, the parameter $\nu$ serves to adjust robustness, and the expected value of the latent variable $\tau$ can be used for outlier detection. For detailed explanations and proofs, please refer to \cite{zhao2023-rfpca}.
	
	
	\subsection{Model formulation}\label{sec:tbfa.model}
	Based on the \mt distribution in \refdef{def.matrix-t}, we propose our \tbfa. Its latent variable model is defined as follows. 
	\begin{IEEEeqnarray}{rCl}\label{eqn:tbfa}
		\hskip-1.5em\left\{
		\begin{array}{l}
			\bX=\bC\bZ\bR'+\bW+\bC\beps_r+\beps_c\bR'+\beps,\quad \mbox{with}\\ 
			\bZ|\tau \sim \cN_{q_c,q_r}(\bo,\bI/\tau,\bI),\,\,\beps_r|\tau \sim \cN_{q_c,d_r}(\bo,\bI/\tau,\bPsi_r),\\
			\beps_c|\tau \sim \cN_{d_c,q_r}(\bo,\bPsi_c/\tau,\bI),\,\,\beps|\tau \sim \cN_{d_c,d_r}(\bo,\bPsi_c/\tau,\bPsi_r),\\
			\tau \sim \cG(\nu/2,\nu/2),
		\end{array}
		\right.
	\end{IEEEeqnarray}
	where $\bW\in\R^{d_c\times d_r}$ is the mean matrix, and $\bC\in\R^{d_c\times q_c}$ and $\bR\in\R^{d_r\times q_r}$ are the column and row factor loading matrices, respectively. $\bZ\in\R^{q_c\times q_r}$ represents the latent factor matrix, $\beps_c\in\R^{d_c\times q_r}$, $\beps_r\in\R^{q_c\times d_r}$, and $\beps\in\R^{d_c \times d_r}$ are the column, row, and common error terms, respectively. Given $\tau$, it is assumed that $\bZ$, $\beps_c$, $\beps_r$, and $\beps$ are independent of each other. The positive matrices $\bPsi_c=\diag(\psi_{c,1},\psi_{c,2},\dots,\psi_{c,d_c})$ and $\bPsi_r=\diag(\psi_{r,1},\psi_{r,2},\dots,\psi_{r,d_r})$.
	
	Similar as BFA in \refs{sec:bfa}, the numbers of free parameters for $\bC$ and $\bR$ are given by $d_cq_c-q_c(q_c-1)/2$ and $d_rq_r-q_r(q_r-1)/2$, respectively. In summary, the total number of free parameters for the \emph{t}BFA model is $\cD=d_c(q_c+1)-q_c(q_c-1)/2 +d_r(q_r+1)-q_r(q_r-1)/2 + d_cd_r + 1$. In addition, the maximum numbers of column factors $q_{c,max}, q_{r,max}$ should satisfy $q_{c,max} \leqslant d_c + (1 - \sqrt{1 + 8d_c})/2$, and $q_{r,max} \leqslant d_r + (1 - \sqrt{1 + 8d_r})/2$.
	
	From \refe{eqn:tbfa}, we can obtain $$\bC\bZ\bR'|\tau\sim\cN_{d_c,d_r}(\bo,\bC\bC'/\tau,\bR\bR'),\,\, \bC\beps_r|\tau\sim\cN_{d_c,d_r}(\bo,\bC\bC'/\tau,\bPsi_r),\,\, \beps_c\bR'|\tau\sim\cN_{d_c,d_r}(\bo,\bPsi_c/\tau,\bR\bR').$$
	Consequently, we have
	\begin{IEEEeqnarray}{rCl}
		\hskip-2em\bX|\tau\sim\cN_{d_c,d_r}(\bW,\bSig_c/\tau,\bSig_r),\quad \bX\sim Mt_{d_c,d_r}(\bW,\bSig_c,\bSig_r,\nu),\label{eqn:dist.X}
	\end{IEEEeqnarray}
	where
	\begin{equation}
		\bSig_c=\bC\bC'+\bPsi_c,\quad
		\bSig_r=\bR\bR'+\bPsi_r.\label{eqn:tbfa.Sig}
	\end{equation}
	
	To better understand \tbfa model, following \cite{zhao2012-bppca, zhao2023-BFA}, we introduce two latent matrices, $\bY^{r}$ and $\bY_{\beps}^r$, and rewrite the model in \refe{eqn:tbfa} as \refe{eqn:tbfa.c}. Model \refe{eqn:tbfa.c} can be viewed as a two-stage representation of the model \refe{eqn:tbfa}. From a projection perspective, the data matrix $\bX$ is first projected onto $\bY^r$ along the column direction, followed by the projection of $\bY^r$ and the residual $\bY_{\beps}^r$ onto the latent factors $\bZ$ and error terms $\beps_c$ along the row direction. 
	
	\begin{center}
		\begin{minipage}{0.45\textwidth}
			\begin{equation}\label{eqn:tbfa.c}
				\left\{
				\begin{array}{ll}
					\bX=\bC\bY^r+\bW+\bY^r_\epsilon, & \hbox{} \\
					\bY^r=\bZ\bR'+\beps_r, & \hbox{}\\
					\bY^r_\epsilon=\beps_c\bR'+\beps. & \hbox{}\\
				\end{array}
				\right.
			\end{equation}
		\end{minipage}
		\hfill
		\begin{minipage}{0.45\textwidth}
			\begin{equation}\label{eqn:tbfa.r}
				\left\{
				\begin{array}{ll}
					\bX=\bY^c\bR'+\bW+\bY^c_\epsilon, & \hbox{} \\
					\bY^c=\bC\bZ+\beps_c, & \hbox{}\\
					\bY^c_\epsilon=\bC\beps_r+\beps. & \hbox{}\\
				\end{array}
				\right.
			\end{equation}
		\end{minipage}
	\end{center}
	
	The following probability distributions can be obtained from \tbfa model \refe{eqn:tbfa}:
	\begin{IEEEeqnarray}{rCl}
		&&\bY^r|\tau \sim \cN_{q_c,d_r}(\bo,\bI/\tau,\bSig_r), \quad \bY^r_\epsilon|\tau \sim \cN_{d_c,d_r}(\bo,\bPsi_c/\tau,\bSig_r), \nonumber \\
		&&\bY^c|\tau \sim \cN_{d_c,q_r}(\bo,\bSig_c/\tau,\bI), \quad \bY^c_\epsilon|\tau \sim \cN_{d_c,d_r}(\bo,\bSig_c/\tau,\bPsi_r), \nonumber \\
		&&\bZ|\tau,\bY^r \sim \cN_{q_c,q_r}(\bY^r\bPsi_r^{-1}\bR\bM_r^{-1},\bI/\tau,\bM_r^{-1}), \label{eqn:dist.Z.Ytau}\\
		&&\bY^r|\bX,\tau \sim \cN_{q_c,d_r}(\bM_c^{-1}\bC'\bPsi_c^{-1}(\bX-\bW),\bM_c^{-1}/\tau,\bSig_r ), \label{eqn:dist.Yr.Xtau}\\
		&&\bY^c|\bX,\tau \sim \cN_{d_c,q_r}((\bX-\bW)\bPsi_r^{-1}\bR\bM_r^{-1},\bSig_c/\tau,\bM_r^{-1}), \label{eqn:dist.Yc.Xtau}\\
		&&\tau|\bX\sim\cG\left((\nu+d_cd_r)/2,(\nu+\delta_{\bX}(\btheta))/2 \right),\label{eqn:dist.tau.X}	
	\end{IEEEeqnarray}
	where $\bSig_c$ and $\bSig_r$ are given in \refe{eqn:tbfa.Sig} and
	\begin{equation}
		\bM_c = \bC'\bPsi_c^{-1}\bC + \bI,\quad
		\bM_r = \bR'\bPsi_r^{-1}\bR + \bI,\label{eqn:tbfa.Mc.Mr}
	\end{equation}
	\begin{equation}
		\delta_{\bX}(\btheta)=\tr\{\bSig_c^{-1}(\bX-\bW)\bSig_r^{-1}(\bX-\bW)'\}.\label{eqn:mdist.X}
	\end{equation}
	
	\subsection{Model identifiability}\label{sec:tbfa.identi}
	Similar as BPPCA and BFA \citep{zhao2012-bppca,zhao2023-BFA}, there are two parameter indeterminacies in \tbfa. One indeterminacy arises from the rotational invariance in FA models, and the other results from the separable covariance structure of matrix-variate distributions.
	
	(i) The factor loading matrices $\bC$ (reps. $\bR$) in \tbfa can be determined uniquely up to an orthogonal rotation. For example, let $\bP_c$ and $\bP_r$ be arbitrary orthogonal matrices of suitable dimensions. Model \refe{eqn:tbfa} remains invariant if we replace $(\bC,\bZ)$ by $(\bC\bP_c,\bP_c^{'}\bZ)$ and/or replace $(\bZ,\bR)$ by $(\bZ\bP_r,\bR\bP_r)$.
	
	(ii) The parameters $(\bC, \bPsi_c)$ and $(\bR, \bPsi_r)$ can be uniquely determined up to a constant. It is easy to verify that for any $a>0$, model \refe{eqn:tbfa} is invariant if we substitute $(\sqrt{a}\cdot\bC, a\cdot\bPsi_c)$ and $(\bR/\sqrt{a}, \bPsi_r/a)$ into \refe{eqn:tbfa}.
	
	To ensure the uniqueness of parameters when investigating the accuracy of estimators in \refs{sec:acc}, for (i), following \cite{fosdick2014-sfa}, we constrain $\bC$ and $\bR$ to lower triangular matrices, i.e., $c_{11}=0$, $r_{11}=0$. For (ii), it is common practice to constrain the first element of $\bPsi_c$ to be 1, i.e.,  $\psi_{c,1}=1$ \citep{srivastava2008}.

	\section{ Maximum likelihood inference of \tbfa}\label{sec:tbfa.infer}
	
	\subsection{Maximum likelihood (ML) estimation of tBFA}\label{sec:tbfa.mle}
	To obtain the ML estimates of the parameters $\btheta=\{\bW,\bC,\bPsi_c,\bR,\bPsi_r,\nu\}$ in \tbfa model, we develop two algorithms in this section. Given a set of observations $\cX=\{\bX_n\}_{n=1}^N$, the log-likelihood function for the observed data (ignoring constant terms) is
	\begin{IEEEeqnarray}{rCl}\label{eqn:tbfa.like}
		\cL(\btheta|\cX)& =& \sum\nolimits_{n=1}^N\ln{\{p(\bX_n)\}} =-\frac12\sum\nolimits_{n=1}^N \left\{(\nu+d_cd_r)\ln{(\nu+\delta_{\bX_n} (\btheta))}+d_r\ln|\bSig_c|\right.\nonumber\\ &&\hskip-0.1em\left.+\>d_c\ln|\bSig_r|\right\}+N\left\{\ln{\Gamma(\frac{\nu+d_cd_r}2)}-\ln{\Gamma(\frac{\nu}2)}+\frac{\nu}2\ln{\nu}\right\}.
	\end{IEEEeqnarray}
	where $\delta_{\bX_n}(\btheta)$ is defined in \refe{eqn:mdist.X}.
	
	We will use EM-type algorithms to maximize $\cL$ in \refe{eqn:tbfa.like}. It is widely acknowledged that the convergence speed of the EM algorithm depends on the proportion of missing information in complete data, and the basic EM algorithm may converge very slowly \citep{meng-aecm}. For practical applications, developing algorithms with fast convergence or reduced time consumption is crucial. Therefore, we propose two EM-type algorithms. The first is the Expectation Conditional Maximization Either (ECME) algorithm \citep{chuanhai_ecme}, which requires less missing information and is expected to converge faster. The second is the Alternating Expectation Conditional Maximization (AECM) algorithm \citep{meng-aecm}, which requires more missing information but has lower per-iteration complexity.
	
	\subsubsection{The ECME algorithm}\label{sec:ECME}
	ECME algorithm consists of an E-step and four CM-steps. In each CM-step, the maximization is performed only with respect to (w.r.t.) a subset of parameters, while keeping the others fixed. The parameters $\btheta$ are partitioned into $\btheta_1=\bW$, $\btheta_2=(\bC, \bPsi_c)$, $\btheta_3=(\bR, \bPsi_r)$, and $\btheta_4=\nu$.
	
	\noindent
	{\bf E-step:} Considering $\cT=\{\tau_n\}_{n=1}^N$ as missing data, the log-likelihood function of the complete data is $\sum\nolimits_{n=1}^N\ln\left\{p(\bX_n|\tau_n)\right.\\ \left.\cdot\>p(\tau_n)\right\}$, and taking its expectation w.r.t. the conditional distribution $p(\cT|\cX,\btheta)$ yields (ignoring constant terms): 
	\begin{IEEEeqnarray}{rCl}
		Q_1(\btheta|\cX)&=&N\left\{\frac\nu2\ln{\frac\nu2}-\ln{\Gamma(\frac\nu2)}-\frac{d_r}2\ln|\bSig_c|-\frac{d_c}2\ln|\bSig_r|\right\}\nonumber\\
		&&\hskip-2em+\>\sum\nolimits_{n=1}^N\left\{\frac\nu2(\bbE[\ln{\tau_n|\bX_n}]-\bbE[\tau_n|\bX_n])-\bbE[\tau_n|\bX_n]\delta_{\bX_n}(\btheta)\right\},\label{eqn:tbfa.Q}
	\end{IEEEeqnarray}
	where
	\begin{equation}
		\tilde{\tau}_n=\bbE[\tau_n|\bX_n]= \frac{\nu+d_cd_r}{\nu+\delta_{\bX_n}(\btheta)},\quad \bbE[\ln\tau_n|\bX_n]= \psi(\frac{\nu+d_cd_r}{2})-\ln(\frac{\nu+\delta_{\bX_n}(\btheta)}{2}),\label{eqn:Etau.X}
	\end{equation}
	and $\psi(x)=d\ln(\Gamma(x))/dx$ is the digamma function.

	\noindent
	{\bf CM-step 1:} Given $(\btheta_2,\btheta_3,\btheta_4)$, maximizing $Q_1$ in \refe{eqn:tbfa.Q} w.r.t. $\btheta_1=\bW$ results in
	\begin{equation}
		\btW=\frac1{\sum\nolimits_{n=1}^N\tilde{\tau}_n}\sum\nolimits_{n=1}^N\tilde{\tau}_n\bX_n. \label{eqn:ecme.W}
	\end{equation}
	
	\noindent
	{\bf CM-step 2:} With $(\bttheta_1,\btheta_3,\btheta_4)$ fixed, maximize $Q_1$ in \refe{eqn:tbfa.Q} w.r.t. $\btheta_2$. This leads to the simplification of $Q_1$ to
	\begin{equation}\label{eqn:tbfa.Qc}
		Q_{1c}(\btheta_2|\cX)=-\frac{Nd_r}{2}\left\{\ln|\bSig_c| +\tr(\bSig_c^{-1}\bS_c)\right\},
	\end{equation}
	where the robust sample column covariance matrix is 
	\begin{IEEEeqnarray}{rCl}
		\bS_c &=& \frac{1}{Nd_r}\sum\nolimits_{n=1}^N \tilde{\tau}_n (\bX_n-\btW)\bSig_r^{-1}(\bX_n-\btW)'.\label{eqn:tbfa.Sc}
	\end{IEEEeqnarray}
	Let the normalized column covariance matrix be
	\begin{equation}\label{eqn:tbfa.Scn}
		\btS_c = \bPsi_c^{-\frac{1}{2}} \bS_c \bPsi_c^{-\frac{1}{2}}.
	\end{equation}	        
	Given $\bPsi_c$, the maximization of $Q_{1c}$ w.r.t. the parameters $\bC$ yields
	\begin{equation}\label{eqn:ecme.C}
		\btC = \bPsi_c^{\frac{1}{2}} \bU_{q_c'} (\bLmd_{q_c'}-\bI)^{\frac{1}{2}}\bV_c,
	\end{equation}
	where $\bU_{q_c'} = (\bu_{c,1},\bu_{c,2}, \dots, \bu_{c,q_c'})$, $\bLmd_{q_c'} = \diag(\lambda_{c,1}, \lambda_{c,2}, \dots, \lambda_{c,q_c'})$, and $\bV_c \in \R^{q_c' \times q_c}$ is an orthogonal matrix satisfying $\bV_c \bV_c'=\bI$. The pairs $(\lambda_{c,i}, \bu_{c,i})$'s represent eigenvalue-eigenvector pairs of $\btS_c$ in \refe{eqn:tbfa.Scn}, arranged in decreasing order, $\lambda_{c,1} \geqslant \lambda_{c,2} \geqslant \dots \geqslant \lambda_{c,d_c}$. The optimal $q_c'$ is determined as follows: if $\lambda_{c,q_c} \geq 1$, then $q_c' = q_c$; otherwise, $q_c'$ is the unique integer satisfying $\lambda_{c,q_c'} \geq 1 \geq \lambda_{c,q_c'+1}$. 
	
	Given $\btC$, maximize $Q_{1c}$ w.r.t. the parameters $\bPsi_c$. \cite{zhao2008-efa} propose updating each diagonal element of $\bPsi_c$ sequentially. The advantage of this approach is that each entry has a closed-form update. Let $\bPsi_{c,i} = \diag(\tpsi_{c,1}, \dots,\tpsi_{c,i-1}, \psi_{c,i},\psi_{c,i+1}, \dots, \psi_{c,d_c})$. By assuming that $\bPsi_c$ is positive definite, we can choose an arbitrary small number $\eta>0$ and assume $\tpsi_{c,i} \geqslant \eta$. Let $\bC^{*} = \bPsi_c^{-1/2}\btC$, and $\be_{i}$ be the $i$-th column of the identity matrix $\bI_{d_c}$,
	\begin{equation} \label{eqn:B}
		\bB_{i} = \sum_{k=1}^{i-1} \tilde{\omega}_{k}\be_{k}\be_{k}^{\prime} + \bI + \bC^{*}\bC^{*\prime},
	\end{equation}
	where $\tilde{\omega}_i$ in \refe{eqn:B} can be obtained by
	\begin{equation} \label{eqn:omega}
		\begin{array}{l}
			\tilde{\omega}_i = \tpsi_{c,i}/\psi_{c,i} - 1.
		\end{array}
	\end{equation}
	Suppose $\bb_{k}$ is the $k$-th column vector of $\bB_{i}^{-1}$, $b_{kk}$ represents the $(k,k)$-th element of $\bB_{i}^{-1}$. $\tpsi_{c,i}$ is then given by
	\begin{equation} \label{eqn:psi_i}
		\tpsi_{c,i} = \max\left\{\left[b_{ii}^{-2}\left(\bb_i'\bS_c\bb_i - b_{ii}\right)+1\right]\psi_{c,i},\eta\right\}.
	\end{equation}
	By \refe{eqn:omega} and \refe{eqn:psi_i}, we know $\tpsi_{c,k} \geqslant \eta$, $\tilde{\omega}_{k} \geq -1,k=1,\dots,i-1$. Thus, $\bB_{i}$ is invertible and $\tpsi_{c,i}$ in \refe{eqn:psi_i} can always be computed. A pseudocode for the update of $\bPsi_c$ is available in \cite{zhao2014-fa-auto}. 
	
	\noindent
	{\bf CM-step 3:} Given $(\bttheta_1,\bttheta_2,\btheta_4)$, maximize $Q_1$ in \refe{eqn:tbfa.Q} w.r.t. $\btheta_3$. $Q_1$ can be simplified to
	\begin{equation}\label{eqn:tbfa.Qr}
		Q_{1r}(\btheta_3)=-\frac{Nd_c}{2}\left\{\ln |\bSig_r| +\tr\left(\bSig_r^{-1}\bS_r\right)\right\},
	\end{equation}
	where the robust sample row covariance matrix is
	\begin{equation}\label{eqn:tbfa.Sr}
		\bS_r = \frac{1}{Nd_c}\sum\nolimits_{n=1}^N\tilde{\tau}_n (\bX_n-\btW)'\btSig_c^{-1}(\bX_n-\btW).
	\end{equation}
	Let the normalized row covariance matrix be 
	\begin{equation}\label{eqn:tbfa.Srn}
		\btS_r=\bPsi_r^{-1/2} \bS_r \bPsi_r^{-1/2}.
	\end{equation}
	Given $\bPsi_r$, maximize $Q_{1r}$ w.r.t. $\bR$, yielding 
	\begin{equation}\label{eqn:ecme.R}
		\btR = \bPsi_r^{\frac{1}{2}} \bU_{q_r^{'}} (\bLmd_{q_r^{'}}-\bI)^{\frac{1}{2}}\bV_r,
	\end{equation}	
	where $\bU_{q_r'}$, $\bLmd_{q_r'}$, and $\bV_r$ are defined similarly to the counterparts of CM-step 2, but computed using $\bS_r$ from \refe{eqn:tbfa.Srn}.
	
	
	\noindent
	{\bf CM-step 4:} Given $(\bttheta_1,\bttheta_2,\bttheta_3)$, maximize the likelihood $\cL$ in \refe{eqn:tbfa.like} w.r.t. $\btheta_4=\nu$. Set the derivative of $\cL$ w.r.t. $\nu$ to zero, i.e., $\cL'(\nu)=0$, yielding the following equation:
	\begin{IEEEeqnarray}{rCl}
		\hskip-2em\cL'(\nu)=-\psi(\frac{\nu}2)+\ln\frac{\nu}2+1+\psi(\frac{\nu+d_cd_r}2)-\ln{(\frac{\nu+d_cd_r}2)}+\frac1N\sum\nolimits_{n=1}^N\left[\ln{\left(\frac{\nu+d_cd_r}{\nu+\tilde{\delta}_{\bX_n}}\right)}-\frac{\nu+d_cd_r}{\nu+\tilde{\delta}_{\bX_n}}\right]=0,\label{eqn:ecme.nu}
	\end{IEEEeqnarray}
	where $\tilde{\delta}_{\bX_n}=\delta_{\bX_n}(\bttheta_1,\bttheta_2,\bttheta_3)$, and $\delta_{\bX_n}(\btheta)$ is defined in \refe{eqn:mdist.X}. Solving equation \refe{eqn:ecme.nu} gives $\tilde{\nu}$. This is only a one-dimensional optimization problem and can be solved by the bisection method as suggested in \cite{liu-tdist}. 
	
	\subsubsection{The AECM algorithm}\label{sec:AECM}
	In this subsection, we develop the AECM algorithm for \tbfa, which consists of multiple cycles within a single iteration and each cycle has its own E-step and CM-step. We partition the parameters to be estimated into $\btheta_1=(\bW,\nu)$, $\btheta_2=(\bC, \bPsi_c)$, and $\btheta_3=(\bR, \bPsi_r)$. The AECM algorithm for \tbfa comprises three cycles, with the following details.
	
	
	1) {\bf Cycle 1:} In the E-step, $(\cX, \cT)=\{\bX_n, \tau_n\}_{n=1}^N$ is treated as complete data. In the CM-step, with $\btheta_2$ and $\btheta_3$ given, maximize the expected log-likelihood of complete data w.r.t. $\btheta_1$. 
	
	\noindent
	{\bf E-step 1:} It is the same as the E-step in the ECME algorithm.   
	
	\noindent
	{\bf CM-Step 1:} Given $\btheta_2$ and $\btheta_3$, maximizing $Q_1$ in \refe{eqn:tbfa.Q} w.r.t. $\btheta_1$ results in $\btW$ in \refe{eqn:ecme.W}, and it leads to the solution for $\tilde{\nu}$ in the following equation \refe{eqn:aecm.nu}.
	\begin{equation}
		-\psi(\frac{\nu}2)+\ln\frac{\nu}2+1+\psi(\frac{\nu+d_cd_r}2)-\ln{(\frac{\nu+d_cd_r}2)} +\frac1N\sum\nolimits_{n=1}^N\left(\ln{\tilde{\tau}_n}-\tilde{\tau}_n\right)=0,\label{eqn:aecm.nu}
	\end{equation} 
	where $\tilde{\tau}_n$ is computed using \refe{eqn:Etau.X}.
	
	2) {\bf Cycle 2:} The E-step treats $(\cX, \cY^r, \cT)=\{\bX_n, \bY_n^r,\tau_n\}_{n=1}^N$ as the complete data set, and the CM-step, with $\bttheta_1$ and $\btheta_3$ given, maximizes the expected complete data log-likelihood w.r.t. $\btheta_2$. 
	
	\noindent
	{\bf E-step 2:} The log-likelihood of the complete data is 
	\begin{equation*}\label{eqn:aecm.like.com.c}
		\cL_c^c(\btheta_2|\cX,
		\cY^r, \cT)=\sum\nolimits_{n=1}^N\ln\{p(\bX_n|\bY^r_n,\tau_n)p(\bY^r_n|\tau_n)p(\tau_n)\}.
	\end{equation*}
	Given $\bttheta_1, \btheta_2, \btheta_3$, calculate the expectation of $\cL_c^c$ w.r.t. the distribution $p(\cY^r,\cT|\cX;\bttheta_1,\btheta_2,\btheta_3)$, yielding
	\begin{equation}
		Q_c(\btheta_2)=-\frac{1}{2}\sum\nolimits_{n=1}^N \left\{d_r \ln|\bPsi_c|+\tr\left\{\bbE\big[\tau_n\bPsi_c^{-1}(\bX_n-\bC\bY_n^r-\btW)\bSig_r^{-1}(\bX_n-\bC\bY_n^r-\btW)'|\bX_n\big]\right\}\right\}.\label{eqn:aecm.Qc}
	\end{equation}
	By \refe{eqn:dist.Yr.Xtau} and \refe{eqn:dist.tau.X}, we obtain 
	\begin{IEEEeqnarray}{rCl}
		\bbE[\tau_n\bY_n^{r}|\bX_n]&=&\bbE[\tau_n\bbE[\bY_n^r|\bX_n,\tau_n]|\bX_n]=\tilde{\tau}_n\bbE[\bY_n^r|\bX_n]=\tilde{\tau}_n\bM_c^{-1}\bC'\bPsi_c^{-1}(\bX_n-\btW),\label{eqn:EtauYr}\\
		\bbE[\tau_n\bY_n^r\bSig_r^{-1}{\bY_n^{r}}'|\bX_n]&=&\bbE[\tau_n\bbE(\bY_n^r\bSig_r^{-1}{\bY_n^{r}}'|\bX_n,\tau_n)|\bX_n]=d_r\bM_c^{-1}+\tilde{\tau}_n\bbE[\bY_n^r|\bX_n]\bSig_r^{-1}\bbE[{\bY_n^{r}}'|\bX_n],\label{eqn:EtauYrYr'}
	\end{IEEEeqnarray}
	where the computation of $\tilde{\tau}_n$ is based on the formula in \refe{eqn:Etau.X}, but with the substitution of $\btheta_1$ by $\bttheta_1$.

	\noindent
	{\bf CM-step 2:} Given $\bttheta_1$ and $\btheta_3$, maximize $Q_c$ in \refe{eqn:aecm.Qc} w.r.t. $\btheta_2$ yields
	\begin{IEEEeqnarray}{rCl}
		\btC &=& \left\{\sum\nolimits_{n=1}^N(\bX_n-\btW)\bSig_r^{-1} \bbE[\tau_n{\bY_n^r}'|\bX_n] \right\} \left\{\sum\nolimits_{n=1}^N\bbE[\tau_n\bY_n^r \bSig_r^{-1} {\bY_n^r}'|\bX_n] \right\}^{-1}\label{eqn:aecm.C}\\
		\btPsi_c &=& \frac{1}{N d_r} \diag\left\{\sum\nolimits_{n=1}^N\tilde{\tau}_n(\bX_n-\btW)\bSig_r^{-1}
		(\bX_n-\btW)'- (\bX_n-\btW)\bSig_r^{-1}\bbE[\tau_n{\bY_n^r}'|\bX_n]\btC'\right\}.\label{eqn:aecm.Psic}
	\end{IEEEeqnarray}
	where $\bS_c$ is given by \refe{eqn:tbfa.Sc}.

	3) {\bf Cycle 3:} In the E-step, consider $(\cX, \cY^c, \cT)=\{\bX_n, \bY_n^c,\tau_n\}_{n=1}^N$ as the complete data. In the CM-step, given $\bttheta_1$ and $\bttheta_2$, maximize the expected complete data likelihood w.r.t. $\btheta_3$.
	
	{\bf E-step 3:} The log-likelihood of the complete data is
	\begin{equation*}\label{eqn:aecm.like.com.r}
		\cL_r^c(\btheta_3|\cX,
		\cY^c,\cT)=\sum\nolimits_{n=1}^N\ln{\{p(\bX_n|\bY^c_n,\tau_n)p(\bY^c_n|\tau_n)p(\tau_n)\}}.
	\end{equation*}
	Given $\bttheta_1, \btheta_2, \btheta_3$, calculate the expectation of $\cL_r^c$ w.r.t. the distribution $p(\cY^c,\cT|\cX;\bttheta_1,\bttheta_2,\btheta_3)$ (ignoring constant terms):      
	\begin{equation}
		Q_r(\btheta_3)=-\frac{1}{2}\sum\nolimits_{n=1}^N \left\{d_c \ln|\bPsi_r|+\tr\left\{\bbE\big[\tau_n\btSig_c^{-1}(\bX_n-\bY_n^c\bR'-\btW)\bPsi_r^{-1}(\bX_n-\bY_n^c\bR'-\btW)'|\bX_n\big]\right\}\right\}.\label{eqn:aecm.Qr}
	\end{equation}
	By \refe{eqn:dist.Yc.Xtau} and \refe{eqn:dist.tau.X}, the required expectations are obtained by
	\begin{IEEEeqnarray}{rCl}
		\bbE[\tau_n\bY_n^c|\bX_n]&=&\bbE[\tau_n\bbE[\bY_n^c|\bX_n,\tau_n]|\bX_n]=\tilde{\tau}_n\bbE[\bY_n^c|\bX_n]=\tilde{\tau}_n(\bX_n-\btW)\bPsi_r^{-1}\bR\bM_r^{-1},\label{eqn:EtauYc}\\
		\bbE[\tau_n{\bY_n^c}'\btSig_c^{-1}\bY_n^c|\bX_n]&=&\bbE[\tau_n\bbE({\bY_n^c}'\btSig_c^{-1}\bY_n^c|\bX_n,\tau_n)|\bX_n]=d_c\bM_r^{-1}+\tilde{\tau}_n\bbE[{\bY_n^c}'|\bX_n]\btSig_c^{-1}\bbE[\bY_n^c|\bX_n],\label{eqn:EtauYc'Yc}
	\end{IEEEeqnarray}
	where the computation of $\tilde{\tau}_n$ is based on the formula in \refe{eqn:Etau.X}, but with the substitution of $\btheta_1$ and $\btheta_2$ by $\bttheta_1$ and $\bttheta_2$, respectively.
	
	\noindent
	{\bf CM-step 3:} Given $\bttheta_1$ and $\bttheta_2$, maximizing $Q_r$ in \refe{eqn:aecm.Qr} w.r.t. $\btheta_3$ yields
	\begin{IEEEeqnarray}{rCl}
		\btR&=&\left\{\sum\nolimits_{n=1}^N(\bX_n-\btW)'\btSig_c^{-1}\bbE[\tau_n\bY_n^c|\bX_n]\right\}\left\{\sum\nolimits_{n=1}^N\bbE[\tau_n{\bY_n^c}'\btSig_c^{-1}\bY_n^c|\bX_n]\right\}^{-1},\label{eqn:aecm.R}\\
		\btPsi_r&= &\frac{1}{Nd_c}\diag\left\{
		\sum\nolimits_{n=1}^N
		\tilde{\tau}_n(\bX_n-\btW)'\btSig_c^{-1}(\bX_n-\btW)-(\bX_n-\btW)'\btSig_c^{-1}\bbE[\tau_n\bY_n^c|\bX_n]\btR'\right\}.\label{eqn:aecm.Psir}
	\end{IEEEeqnarray}

	\subsubsection{Acceleration via Parameter Expansion}\label{sec:PX-alg}
	Due to the slow convergence of traditional EM-type algorithms for multivariate \emph{t} distributions, \cite{liu-PXEM} proposed a parameter expansion technique to accelerate them. This technique introduces auxiliary parameters to extend the original model while keeping the observed data model unchanged, and then develops EM-type algorithms on the extended complete data model. In this section, we apply this technique to \tbfa, and develop parameter-expanded ECME (PX-ECME) and PX-AECME algorithms to accelerate their parent ECME in \refs{sec:ECME} and AECM in \refs{sec:AECM}.
	
	After introducing the auxiliary parameter $\alpha$, the extended model is represented as follows:
	\begin{equation}\label{eqn:tbfa.x}
		\hskip-1.2em\left\{
		\begin{array}{l}
			\bX=\bC_\star\bZ{\bR'}_\star+\bW_\star+\bC_\star\beps_r+\beps_c{\bR'}_\star+\beps,\\ 
			\bZ|\tau \sim \cN_{q_c,q_r}(\bo,\bI/\tau,\bI),\beps_c|\tau \sim \cN_{d_c,q_r}(\bo,\bPsi_{c\star}/\tau,\bI),\\
			\beps_r|\tau \sim \cN_{q_c,d_r}(\bo,\bI/\tau,\bPsi_{r\star}),\\ 
			\beps|\tau \sim \cN_{d_c,d_r}(\bo,\bPsi_{c\star}/\tau,\bPsi_{r\star}),\,\,\tau\sim\alpha\cG(\nu_\star/2,\nu_\star/2),
		\end{array}
		\right.
	\end{equation}
	where the parameters are marked with subscript ${}_\star$ to distinguish them from the parameters in the original model \refe{eqn:tbfa}.
	
	The E-step(s) and CM-step(s) in PX-ECME and PX-AECM are similar to those in ECME and AECM, but they are implemented on the extended parameters $(\btheta_\star,\alpha)$. Under the extended model \refe{eqn:tbfa.x}, we obtain
	\begin{IEEEeqnarray}{rCl}
		\bX|\tau&\sim&\cN_{d_c,d_r}(\bW_\star,\bSig_{c\star}/\tau,\bSig_{r\star}), \quad \bX\sim t(\bW_\star,\bSig_{c\star}/\alpha,\bSig_{r\star},\nu_\star),\label{eqn:dist.X.x}\\
		\tau|\bX&\sim&\cG\left(\frac{\nu_\star+d_cd_r}{2},\frac{\nu_\star+\delta_{\bX}(\btheta_\star)}{2\alpha}\right),\label{eqn:dist.tau.X.x}
	\end{IEEEeqnarray}
	where
	\begin{equation}
		\delta_{\bX}(\btheta_\star) = \tr\{(\bSig_{c\star}/\alpha)^{-1}(\bX-\bW_\star)\bSig_{r\star}^{-1}(\bX-\bW_\star)'\}.\label{eqn:mdist.X.x}
	\end{equation}
	By \refe{eqn:dist.X.x}, $\bSig_{c\star}/\alpha$ corresponds to $\bSig_c$, and $(\bW_\star,\bSig_{r\star},\nu_\star)$ corresponds to $(\bW, \bSig_r,\nu)$. Moreover, since $\bSig_{c\star}=\bC_\star\bC_\star^{'}+\bPsi_{c\star}$, we establish the mappings $\bC_\star/\sqrt{\alpha} \leftrightarrow \bC$ and $\bPsi_{c\star}/\alpha \leftrightarrow \bPsi_c$. Through this mapping, the extended parameters $(\btheta_\star,\alpha)$ can be reduced back to the original parameters $\btheta$ at the end of each iteration.
	
	For the PX-ECME algorithm, we observe that, compared to the ECME algorithm, its difference lies in replacing $\bS_c$ in \refe{eqn:tbfa.Sc} and $\bS_r$ in \refe{eqn:tbfa.Sr} with $\bS_c^{*}$ in \refe{eqn:tbfa.Sc.x} and $\bS_r^{*}$ in \refe{eqn:tbfa.Sr.x}. 
	\begin{IEEEeqnarray}{rCl}
		\bS_c^{*}&=&\frac{1}{d_r \sum\nolimits_{n=1}^N\tilde{\tau}_n}\sum\nolimits_{n=1}^N\tilde{\tau}_n (\bX_n-\btW)\bSig_r^{-1}(\bX_n-\btW)' =\bS_{c}N/\sum\nolimits_{n=1}^N\tilde{\tau}_n,\label{eqn:tbfa.Sc.x}\\
		\bS_r^{*}&=&\frac{1}{d_c \sum\nolimits_{n=1}^N\tilde{\tau}_n}\sum\nolimits_{n=1}^N\tilde{\tau}_n (\bX_n-\btW)'\btSig_c^{-1}(\bX_n-\btW)=\bS_{r}N/\sum\nolimits_{n=1}^N\tilde{\tau}_n. \label{eqn:tbfa.Sr.x}
	\end{IEEEeqnarray}
	
	For the PX-AECM algorithm, by model \refe{eqn:tbfa.x}, it can be implemented as follows:
	
	\noindent 1) {\bf Cycle 1:} 
	
	\noindent
	{\bf PXE-step 1:} Compute the required expectations:
	\begin{IEEEeqnarray}{rCl}
		\tilde{\tau}_{n\star}&=&\bbE_\star[\tau_n|\bX_n]=\alpha\frac{\nu_\star+d_cd_r}{\nu_\star+\delta_{\bX_n}(\btheta_\star)}.\label{eqn:Etau.X.x}\\
		\bbE_\star[\ln\tau_n|\bX_n]&=& \psi(\frac{\nu_\star+d_cd_r}{2})-\ln(\frac{\nu_\star+\delta_{\bX_n}(\btheta_\star)}{2\alpha}).\nonumber
	\end{IEEEeqnarray}
	
	\noindent
	{\bf PXCM-step 1:} The updates for $\tilde{\alpha}$ and $\btW_\star$ are as follows:
	\begin{equation}
		\tilde{\alpha}=\sum\nolimits_{n=1}^N\tilde{\tau}_{n\star}/N, \quad \btW_\star=\frac1{\sum\nolimits_{n=1}^N\tilde{\tau}_{n\star}}\sum\nolimits_{n=1}^N\tilde{\tau}_{n\star}\bX_n. \label{eqn:W_star}
	\end{equation}
	The update for $\tilde{\nu_\star}$ is obtained by solving the following equation:
	\begin{equation}
		-\psi(\frac{\nu_\star}2)+\ln(\frac{\nu_\star}{2\tilde{\alpha}})+1+\frac1N\sum\nolimits_{n=1}^N\left(\bbE_\star[\ln\tau_n|\bX_n]-\tilde{\tau}_{n\star}/\tilde{\alpha}\right)=0. \label{eqn:nu_star}
	\end{equation}
	
	\noindent 2) {\bf Cycle 2:}
	
	\noindent
	{\bf PXE-step 2:} Compute the required expectations:
	\begin{IEEEeqnarray}{rCl}
		\bbE_\star[\tau_n\bY_n^{r}|\bX_n]&=&\bbE[\tau_n\bbE[\bY_n^r|\bX_n,\tau_n]|\bX_n]=\tilde{\tau}_{n\star}\bbE[\bY_n^r|\bX_n]=\tilde{\tau}_{n\star}\bM_{c\star}^{-1}\bC_\star^{'}\bPsi_{c\star}^{-1}(\bX_n-\btW_\star),\label{eqn:EtauYr.x}\\
		\bbE_\star[\tau_n\bY_n^r\bSig_{r\star}^{-1}{\bY_n^{r}}'|\bX_n]&=&\bbE[\tau_n\bbE(\bY_n^r\bSig_{r\star}^{-1}{\bY_n^{r}}'|\bX_n,\tau_n)|\bX_n]=d_r\bM_{c\star}^{-1}+\tilde{\tau}_{n\star}\bbE[\bY_n^r|\bX_n]\bSig_{r\star}^{-1}\bbE[{\bY_n^{r}}'|\bX_n],\label{eqn:EtauYrYr'.x}
	\end{IEEEeqnarray}
	where $\tilde{\tau}_{n\star}$ is computed using \refe{eqn:Etau.X.x}, but with the substitution of $(\alpha,\nu_\star,\bW_\star)$ by $(\tilde{\alpha},\tilde{\nu_\star},\btW_\star)$.
	
	\noindent
	{\bf PXCM-step 2:} The updates for $\btC_\star, \btPsi_{c\star}$ are as follows:
	\begin{IEEEeqnarray}{rCl}
		\btC_\star &=& \left\{\sum_{n=1}^{N}(\bX_n-\btW_\star)\bSig_{r\star}^{-1} \bbE_\star[\tau_n{\bY_n^r}'|\bX_n] \right\} \left\{\sum_{n=1}^{N} \bbE_\star[\tau_n\bY_n^r \bSig_{r\star}^{-1} {\bY_n^r}'|\bX_n] \right\}^{-1},\label{eqn:C_star}\\
		\btPsi_{c\star} &=& \frac{1}{N d_r} \diag\left\{\sum_{n=1}^{N}\tilde{\tau}_{n\star}(\bX_n-\btW_\star)\bSig_{r\star}^{-1}
		(\bX_n-\btW_\star)'- (\bX_n-\btW_\star)\bSig_{r\star}^{-1}\bbE_\star[\tau_n{\bY_n^r}'|\bX_n]\btC_\star^{'}\right\}.\label{eqn:Psic_star}
	\end{IEEEeqnarray}
	
	\noindent 3) {\bf Cycle 3:}
	
	\noindent
	{\bf PXE-step 3:} Compute the required expectations:
	\begin{IEEEeqnarray}{rCl}
		&&\bbE_\star[\tau_n\bY_n^c|\bX_n]=\bbE[\tau_n\bbE[\bY_n^c|\bX_n,\tau_n]|\bX_n]=\tilde{\tau}_{n\star}\bbE[\bY_n^c|\bX_n]=\tilde{\tau}_{n\star}(\bX_n-\btW_\star)\bPsi_{r\star}^{-1}\bR_\star\bM_{r\star}^{-1},\label{eqn:EtauYc.x}\\
		&&\bbE_\star[\tau_n{\bY_n^c}'\btSig_{c\star}^{-1}\bY_n^c|\bX_n]=\bbE[\tau_n\bbE({\bY_n^c}'\btSig_{c\star}^{-1}\bY_n^c|\bX_n,\tau_n)|\bX_n]=d_c\bM_{r\star}^{-1}+\tilde{\tau}_{n\star}\bbE[{\bY_n^c}'|\bX_n]\btSig_{c\star}^{-1}\bbE[\bY_n^c|\bX_n],\label{eqn:EtauYc'Yc.x}
	\end{IEEEeqnarray}
	where $\tilde{\tau}_{n\star}$ is computed using \refe{eqn:Etau.X.x}, but with the substitution of $(\alpha, \nu_\star,\bW_\star,\bC_\star,\bPsi_{c\star})$ by $(\tilde{\alpha},\tilde{\nu_\star},\btW_\star,\btC_\star,\btPsi_{c\star})$.
	
	\noindent
	{\bf PXCM-step 3:} The updates for $\btR_\star, \btPsi_{r\star}$ are as follows:
	\begin{IEEEeqnarray}{rCl}
		\btR_\star&=&\left\{\sum_{n=1}^{N}(\bX_n-\btW_\star)'\btSig_{c\star}^{-1}\bbE_\star[\tau_n\bY_n^c|\bX_n]\right\}\left\{\sum_{n=1}^{N}\bbE_\star[\tau_n{\bY_n^c}'\btSig_{c\star}^{-1}\bY_n^c|\bX_n]\right\}^{-1},\label{eqn:R_star}\\
		\btPsi_{r\star}&= &\frac{1}{Nd_c}\diag\left\{
		\sum_{n=1}^{N}\tilde{\tau}_{n\star}(\bX_n-\btW_\star)'\btSig_{c\star}^{-1}(\bX_n-\btW_\star)-(\bX_n-\btW_\star)'\btSig_{c\star}^{-1}\bbE_\star[\tau_n\bY_n^c|\bX_n]\btR_\star'\right\}.\label{eqn:Psir_star}
	\end{IEEEeqnarray}
	
	
	
	\subsection{Matrix factor score}\label{sec:tbfa.score}
	In FA model, the factor scores $\bz$ is a vector and can be computed by $\bbE(\bz|\bx)$ \citep{anderson-mult-3ed,zhao2008-efa}. However, in \emph{t}BFA model, the factor scores $\bZ$ is a matrix, instead of a vector. From \refe{eqn:dist.Z.Ytau} and \refe{eqn:dist.Yr.Xtau}, the matrix factor score $\bbE(\bZ|\bX)$ can be obtained by
	\begin{equation}\label{eqn:reg}
		\bbE\left(\bZ\mid\bX\right) = \bbE\left[\bbE\left(\bZ\mid\bY^r,\tau\right)\mid\bX\right] = \bM_c^{-1} \bC' \bPsi_c^{-1} \left(\bX-\bW\right) \bPsi_r^{-1} \bR \bM_r^{-1}.
	\end{equation}
	

	\subsection{Asymptotic properties and estimating precision of parameter estimates}\label{sec:stderr}
	In this subsection, we discuss the theoretical properties of the ML estimator $\bwTheta$ as the sample size $N\rightarrow\infty$. Under general regularity conditions, $\bwTheta$ can converge in probability to the true value of $\bTheta_0$ and distribution to a multivariate normal distribution with the mean vector $\bTheta_0$ and covariance matrix being the inverse of the Fisher information matrix. Specifically, we have
	\begin{equation}\label{eqn:asym}
		\sqrt{N}(\bwTheta - \bTheta_0) \overset{L}{\rightarrow} \cN_{\cD}(\bo, \bI_1^{-1}(\bTheta_0)),
	\end{equation}
	where $\cD$ represents the dimension of $\bTheta$, which is the number of free parameters in \tbfa. $\bI_1(\bTheta_0)$ denotes the Fisher information matrix evaluated at $\bTheta_0$, and the Fisher information matrix containing $N$ i.i.d. observations can be computed as $\bI_N(\bTheta)=N\bI_1(\bTheta)$. The specific expressions of $\bI_N(\bTheta)$ are provided below.
	
	To measure the accuracies of the ML estimators, following \cite{wang2017automated}, we derive the closed-form expression of $\bI_N(\bTheta)$ that can be used to compute the standard errors of $\bwTheta$. Let $\bmu = \vec(\bW)$. Denote $\btheta_c=\vech(\bSig_c)$ and $\btheta_r=\vech(\bSig_r)$. We calculate the first derivatives of \refe{eqn:tbfa.like} w.r.t. each element of $\bTheta$ to obtain the score vector
	\begin{IEEEeqnarray}{rCl}\label{eqn:score}
		\bs \left(\btheta\vert \cX \right) = \vec\left(\bs^{\bmu},\bs^{\btheta_c},\bs^{\btheta_r},\bs^{\nu}\right).
	\end{IEEEeqnarray}
	The Hessian matrix can be computed using the related second derivatives of the log-likelihood function w.r.t. each element of $\bTheta$
	\begin{IEEEeqnarray}{rCl}\label{eqn:Hessian}
		\bH(\btheta\mid\cX) = \begin{bmatrix}
			\bH^{\bmu\bmu} 				&  \bH^{\bmu\btheta_c}		&\bH^{\bmu\btheta_r} &\bH^{\bmu\nu}\\
			\bH^{\btheta_c\bmu}	&  \bH^{\btheta_c\btheta_c}	& \bH^{\btheta_c\btheta_r}& \bH^{\btheta_c\nu}\\
			\bH^{\btheta_r\bmu}	&  \bH^{\btheta_r\btheta_c}&  \bH^{\btheta_r\btheta_r} & \bH^{\btheta_r\nu}\\
			\bH^{\nu\bmu}	&  \bH^{\nu\btheta_c}&  \bH^{\nu\btheta_r} & \bH^{\nu\nu}
		\end{bmatrix}.
	\end{IEEEeqnarray}
	
	The expressions for each element of \refe{eqn:score} and \refe{eqn:Hessian} can be found in the \refthm{thm:info}. To obtain the Fisher information matrix, the following \refthm{thm:Exp} is required. The Fisher information matrix $\bI_N(\bTheta)=-\sum_{n=1}^{N} \bbE[\bH_n(\btheta\mid \bX_n)]=-\bbE[\bH(\btheta\mid\cX)]$ is detailed in \refthm{thm:info}.
	\begin{thm} \label{thm:Exp}
		For tBFA model \refe{eqn:tbfa}, we have: 
		\begin{IEEEeqnarray*}{rCl}
			(a) \ &&\bbE\left\{\left(\nu+\delta_{\bX_n}(\btheta)\right)^{-1}\right\}=\left(\nu+d_cd_r\right)^{-1};\\
			(b)\ && \bbE\left\{\left(\nu+\delta_{\bX_n}\right)^{-1} \delta_{\bX_n}(\btheta)\right\}=\left(\nu+d_cd_r\right)^{-1} d_cd_r;\\
			(c)\ && \bbE\left\{\nu^2\left(\nu+\delta_{\bX_n}(\btheta)\right)^{-2}\right\}=\left(\nu+d_cd_r\right)^{-1}\left(\nu+d_cd_r+2\right)^{-1} \nu(\nu+2);\\
			(d)\ &&\bbE\left\{\left(\nu+\delta_{\bX_n}(\btheta)\right)^{-1} \beps_n \beps_n^{'}\right\}=\left(\nu+d_cd_r\right)^{-1} \left(\bSig_r \otimes \bSig_c\right);\\
			(e)\ && \bbE\left\{\left(\nu+\delta_{\bX_n}(\btheta)\right)^{-2} \beps_n \beps_n^{ '}\right\}=\left(\nu+d_cd_r\right)^{-1} \left(\nu+d_c d_r+2\right)^{-1}\left(\bSig_r \otimes \bSig_c\right) ; \\
			(f)\ && \bbE\left\{\left(\nu+\delta_{\bX_n}(\btheta)\right)^{-2} \beps_n^{'} \left( \bSig_r^{-1} 
			\dot{\bSig}_r^k \bSig_r^{-1}  \otimes  \bSig_c^{-1} \right) \beps_n \beps_n^{'} \left(\bSig_r^{-1} \otimes\left(\bSig_c^{-1} \dot{\bSig}_c^i \bSig_c^{-1}\right)\right) \beps_n\right\} \\ \ &&=\left(\nu+d_cd_r\right)^{-1}\left(\nu+d_cd_r+2\right)^{-1} 
			\left\{d_cd_r\tr\left(\bSig_r^{-1} \dot{\bSig}_r^k \right) \tr\left(\bSig_c^{-1} \dot{\bSig}_c^i \right) + 2\tr\left(\left(\bSig_r^{-1} \dot{\bSig}_r^k \right) \otimes \left(\bSig_c^{-1} \dot{\bSig}_c^i \right)\right)  \right\};\\
			(g) \ &&\bbE\left\{\left(\nu+\delta_{\bX_n}(\btheta)\right)^{-2} \beps_n^{'} \left( \bSig_r^{-1} \otimes \bSig_c^{-1} \dot{\bSig}_c^i \bSig_c^{-1}   \right) \beps_n \beps_n^{'} \left(\bSig_r^{-1} \otimes\left(\bSig_c^{-1} \dot{\bSig}_c^j \bSig_c^{-1}\right)\right) \beps_n\right\} \\ \ &&=\left(\nu+d_cd_r\right)^{-1}\left(\nu+d_cd_r+2\right)^{-1} 
			\left\{d_r{^2}\tr\left(\bSig_c^{-1} \dot{\bSig}_c^i \right) \tr \left(\bSig_c^{-1} \dot{\bSig}_c^j \right) + 2d_r\tr \left(\bSig_c^{-1} \dot{\bSig}_c^i \bSig_c^{-1} \dot{\bSig}_c^j  \right)  \right\};\\
			(h)\ && \bbE\left\{\left(\nu+\delta_{\bX_n}(\btheta)\right)^{-2} \beps_n^{'} \left( \bSig_r^{-1} \dot{\bSig}_r^k \bSig_r^{-1}  \otimes  \bSig_c^{-1} \right) \beps_n \beps_n^{'} \left(\left(\bSig_r^{-1} \dot{\bSig}_r^s \bSig_r^{-1}\right)  \otimes \bSig_c^{-1} \right) \beps_n\right\} \\ \ &&=\left(\nu+d_cd_r\right)^{-1}\left(\nu+d_cd_r+2\right)^{-1} 
			\left\{d_c{^2}\tr\left(\bSig_r^{-1} \dot{\bSig}_r^k \right) \tr\left(\bSig_r^{-1} \dot{\bSig}_r^s \right) + 2d_c\tr\left(\bSig_r^{-1} \dot{\bSig}_r^k\bSig_r^{-1} \dot{\bSig}_r^s \right)\right\},
		\end{IEEEeqnarray*}
		where $\beps_n=\vec \left( \bX_n-\bW \right)$, 
		$\dot{\bSig}_c^i =\partial \bSig_c/ \partial \btheta_c^i$ with $\btheta_c$ representing the i-th entry of $\btheta_c$, for each $i \in\{1, \ldots, d_c(d_c+1) / 2\}$, \ and $\dot{\bSig}_r^k =\partial \bSig_r/ \partial \btheta_r^k$ with $\btheta_r$ representing the k-th entry of $\btheta_r$, for each 
		$k \in \{1, \ldots, d_r(d_r+1) / 2\}$.
	\end{thm}
\begin{proof}
	Recalling that $\vec\left(\bX_n \right) \sim t_{d_cd_r}(\vec(\bW),\bSig_r \otimes \bSig_c,\nu)$,\ $\delta_{\bX_n}(\btheta)=\tr\{\bSig_c^{-1}(\bX_n-\bW)\bSig_r^{-1}(\bX_n-\bW)'\} \sim d_cd_r F\left(d_cd_r, \nu\right)$ and 
$\nu/(\nu+\delta_{\bX_n}(\btheta)) \sim \operatorname{Beta}\left(\nu / 2, d_cd_r/ 2\right)$.  With reference to these properties and the work by \cite{wang2017automated}, we can readily demonstrate that (a)-(e). Hereafter, we will provide the proof solely for (f).

(f) Let $\theta_c^i$ represent the $i$th entry of $\btheta_c$,
$\theta_r^k$ denote the $k$-th entry of $\btheta_r$, 
$ \dot{\bSig}_c^i=\partial \bSig_c/\partial \theta_c^i$ for $i \in\{1, \ldots, d_c(d_c+1) / 2\}$ and 
$ \dot{\bSig}_r^k=\partial \bSig_r/\partial \theta_r^k$
for $k \in\{1, \ldots, d_r(d_r+1) / 2\}$ throughout the following proof. Due to
\begin{IEEEeqnarray}{rCl}
	\bbE \left\{\frac{\partial^2 \log f\left(\bx_n\right)}{\partial \theta_c^i \partial 
		\theta_r^k}\right\}= -\bbE \left\{\frac{\partial \log f\left(\bx_n\right)}{\partial \theta_c^i} \frac{\partial \log f\left(\bx_n \right)}{\partial \theta_r^k}\right\},  
	\label{eqn:gs}
\end{IEEEeqnarray}
we first derive 
\begin{IEEEeqnarray}{rCl}  
	\nonumber
	\frac{\partial^2 \log f\left(\bx_n\right)}{\partial \theta_c^i \partial \theta_r^k}
	\nonumber
	= &&\frac{1}{2}\left(\nu+d_cd_r \right)
	\frac{\beps_n^{'} \left( \bSig_r^{-1} \dot{\bSig}_r^k \bSig_r^{-1}  \otimes  \bSig_c^{-1} \right) \beps_n \beps_n^{'} \left(\bSig_r^{-1} \otimes \left(\bSig_c^{-1} \dot{\bSig}_c^i \bSig_c^{-1}\right)\right) \beps_n}{\left(\nu+\delta_{\bX_n}(\btheta)\right)^{2}} \\
	&&-\frac{1}{2}\left(\nu+d_cd_r\right) \frac{ \tr \left\{ \left( \left( \bSig_r^{-1} \dot{\bSig}_r^k \bSig_r^{-1}\right)  \otimes  \left( \bSig_c^{-1} \dot{\bSig}_c^i \bSig_c^{-1} \right)\right) \beps_n \beps_n^{'}  \right\}}{\nu+\delta_{\bX_n}(\btheta)}, \label{eqn:Sec}
\end{IEEEeqnarray}
and 
\begin{IEEEeqnarray}{rCl} 
	-\left\{\frac{\partial \log f\left(\bx_n \right)}{\partial \theta_c^i }\right\} \left\{\frac{\partial \log f\left(\bx_n \right)}{\partial \theta_r^k}\right\} = -\frac{1}{4} d_cd_r
	\tr \left(\bSig_c^{-1}\dot{\bSig}_c^i \right) \tr\left(\bSig_r^{-1}\dot{\bSig}_r^k \right)\\
	\nonumber
	&&\hskip-26.5em +\frac{1}{4} d_r \left(\nu+d_cd_r\right) \frac{\beps_n^{'}  \left( \bSig_r^{-1} \dot{\bSig}_r^k \bSig_r^{-1}\otimes  \bSig_c^{-1} \right) \beps_n} {\nu+\delta_{\bX_n}(\btheta)} \tr \left(\bSig_c^{-1} \dot{\bSig}_c^i\right) \\
	\nonumber
	&&\hskip-26.5em +\frac{1}{4} d_c \left(\nu+d_cd_r\right) \frac{\beps_n^{'}  \left( \bSig_r^{-1} \otimes  \left( 
		\bSig_c^{-1} \dot{\bSig}_c^i \bSig_c^{-1} \right)  \right) \beps_n} {v+\delta_{\bX_n}(\btheta)} \tr \left(\bSig_r^{-1} \dot{\bSig}_r^k\right) \\
	&&\hskip-26.5em-\frac{1}{4}\left(\nu+d_cd_r\right)^2\left\{\frac{\beps_n^{'}  \left( \bSig_r^{-1} \otimes  \left( 
		\bSig_c^{-1} \dot{\bSig}_c^i \bSig_c^{-1} \right) \right) \beps_n \beps_n^{'}  \left( \bSig_r^{-1} \dot{\bSig}_r^k \bSig_r^{-1}\otimes  \bSig_c^{-1} \right) \beps_n }{\left(\nu+\delta_{\bX_n}(\btheta)\right)^2}\right\}.
	\label{eqn:Fir} 
\end{IEEEeqnarray}


Taking the expectations for \refe{eqn:Sec} and \refe{eqn:Fir}, we obtain
\begin{IEEEeqnarray}{rCl}  \label{eqn:ESec}
	\nonumber
	\bbE \left\{\frac{\partial^2 \log f\left(\bx_n\right)}{\partial \theta_c^i \partial \theta_r^k}\right\}=&& \frac{1}{2}\left(\nu+d_cd_r \right)
	\bbE \left\{ \frac{\beps_n^{'} \left( \bSig_r^{-1} \dot{\bSig}_r^k \bSig_r^{-1}  \otimes  \bSig_c^{-1} \right) \beps_n \beps_n^{'} \left(\bSig_r^{-1} \otimes \left(\bSig_c^{-1} \dot{\bSig}_c^i \bSig_c^{-1}\right)\right) \beps_n}{\left(\nu+\delta_{\bX_n}(\btheta)\right)^{2}}\right\} \\
	\nonumber
	&& -\frac{1}{2}\left(\nu+d_cd_r\right) \bbE \left\{\frac{ \tr \left\{ \left( \left( \bSig_r^{-1} \dot{\bSig}_r^k \bSig_r^{-1}\right)  \otimes  \left( \bSig_c^{-1} \dot{\bSig}_c^i \bSig_c^{-1} \right)\right) \beps_n \beps_n^{'}  \right\}}{\nu+\delta_{\bX_n}(\btheta)}\right\}  \\
	\nonumber
	&& \hskip-8em =\frac{1}{2}\left(\nu+d_cd_r \right)
	\bbE \left\{ \frac{\beps_n^{'} \left( \bSig_r^{-1} \dot{\bSig}_r^k \bSig_r^{-1}  \otimes  \bSig_c^{-1} \right) \beps_n \beps_n^{'} \left(\bSig_r^{-1} \otimes \left(\bSig_c^{-1} \dot{\bSig}_c^i \bSig_c^{-1}\right)\right) \beps_n}{\left(\nu+\delta_{\bX_n}(\btheta)\right)^{2}}\right\} = -\frac{1}{2} \tr \left\{ \left( \bSig_r^{-1} \dot{\bSig}_r^k \right)  \otimes  \left( \bSig_c^{-1} \dot{\bSig}_c^i  \right) \right\},
\end{IEEEeqnarray}

\begin{IEEEeqnarray}{rCl} \label{eqn:EFir}
	\nonumber
	&& -\bbE\left[\left\{\frac{\partial \log f\left(\bx_n \right)}{\partial \theta_c^i }\right\}\left\{\frac{\partial \log f\left(\bx_n \right)}{\partial \theta_r^k}\right\}\right]\\
	\nonumber
	=&&-\frac{1}{4} d_cd_r
	\tr \left(\bSig_c^{-1}\dot{\bSig}_c^i \right) \tr\left(\bSig_r^{-1}\dot{\bSig}_r^k \right)+\frac{1}{4} d_r \left(\nu+d_cd_r\right) \bbE \left\{ \frac{\beps_n^{'}  \left( \bSig_r^{-1} \dot{\bSig}_r^k \bSig_r^{-1}\otimes  \bSig_c^{-1} \right) \beps_n} {\nu+\delta_{\bX_n}(\btheta)} \right\} \tr \left(\bSig_c^{-1} \dot{\bSig}_c^i\right) \\
	\nonumber
	&& +\frac{1}{4} d_c \left(\nu+d_cd_r\right)\bbE \left\{ \frac{\beps_n^{'}  \left( \bSig_r^{-1} \otimes  \left( 
		\bSig_c^{-1} \dot{\bSig}_c^i \bSig_c^{-1} \right)  \right) \beps_n} {\nu+\delta_{\bX_n}(\btheta)} \right\} \tr \left(\bSig_r^{-1} \dot{\bSig}_r^k\right) \\
	\nonumber
	&& -\frac{1}{4}\left(\nu+d_cd_r\right)^2 \bbE \left\{\frac{\beps_n^{'}  \left( \bSig_r^{-1} \otimes  \left( 
		\bSig_c^{-1} \dot{\bSig}_c^i \bSig_c^{-1} \right) \right) \beps_n \beps_n^{'}  \left( \bSig_r^{-1} \dot{\bSig}_r^k \bSig_r^{-1}\otimes  \bSig_c^{-1} \right) \beps_n }{\left(\nu+\delta_{\bX_n}(\btheta)\right)^2}\right\} \\
	\nonumber
	=&& -\frac{1}{4} d_cd_r
	\tr \left(\bSig_c^{-1}\dot{\bSig}_c^i \right) \tr\left(\bSig_r^{-1}\dot{\bSig}_r^k \right) +\frac{1}{4} d_cd_r \tr \left(\bSig_c^{-1} \dot{\bSig}_c^i\right) \tr \left( \bSig_r^{-1} \dot{\bSig}_r^k \right)  
	+\frac{1}{4} d_cd_r \tr \left(\bSig_c^{-1} \dot{\bSig}_c^i\right) \tr \left( \bSig_r^{-1} \dot{\bSig}_r^k \right)  \\
	&& -\frac{1}{4}\left(\nu+d_cd_r\right)^2 \bbE \left\{\frac{\beps_n^{'}  \left( \bSig_r^{-1} \otimes  \left( 
		\bSig_c^{-1} \dot{\bSig}_c^i \bSig_c^{-1} \right) \right) \beps_n \beps_n^{'}  \left( \bSig_r^{-1} \dot{\bSig}_r^k \bSig_r^{-1}\otimes  \bSig_c^{-1} \right) \beps_n }{\left(\nu+\delta_{\bX_n}(\btheta)\right)^2}\right\} 
\end{IEEEeqnarray}

Additional details regarding the terms in \refe{eqn:ESec} and \refe{eqn:EFir} are
\begin{equation*}
	\begin{aligned}
		& \hskip-2.5em-\frac{1}{2}\left(\nu+d_cd_r\right) \bbE \left\{\frac{ \tr \left\{ \left( \left( \bSig_r^{-1} \dot{\bSig}_r^k \bSig_r^{-1}\right)  \otimes  \left( \bSig_c^{-1} \dot{\bSig}_c^i \bSig_c^{-1} \right)\right) \beps_n \beps_n^{'}  \right\}}{\nu+\delta_{\bX_n}(\btheta)}\right\} \\
		=& -\frac{1}{2}\left(\nu+d_cd_r\right) \bbE \left[ \tr \left\{\frac{  \left( \left( \bSig_r^{-1} \dot{\bSig}_r^k \bSig_r^{-1}\right)  \otimes  \left( \bSig_c^{-1} \dot{\bSig}_c^i \bSig_c^{-1} \right)\right) \beps_n \beps_n^{'}  }{\nu+\delta_{\bX_n}(\btheta)}\right\}\right] \\
		=& -\frac{1}{2}\left(\nu+d_cd_r\right)  \tr \left\{  \left( \left( \bSig_r^{-1} \dot{\bSig}_r^k \bSig_r^{-1}\right)  \otimes  \left( \bSig_c^{-1} \dot{\bSig}_c^i \bSig_c^{-1} \right)\right) 
		\bbE \left(\frac{\beps_n \beps_n^{'}  }{\nu+\delta_{\bX_n}(\btheta)}\right)\right\} \\
		=& -\frac{1}{2}\left(\nu+d_cd_r\right)  \tr \left\{  \left( \left( \bSig_r^{-1} \dot{\bSig}_r^k \bSig_r^{-1}\right)  \otimes  \left( \bSig_c^{-1} \dot{\bSig}_c^i \bSig_c^{-1} \right)\right) 
		\left(\frac{\bSig_r \otimes  \bSig_c }{\nu+d_cd_r}\right)\right\} = -\frac{1}{2} \tr \left\{ \left( \bSig_r^{-1} \dot{\bSig}_r^k \right)  \otimes  \left( \bSig_c^{-1} \dot{\bSig}_c^i  \right) \right\},
	\end{aligned}
\end{equation*}

\begin{IEEEeqnarray*}{rCl}
	&& \hskip-1.5em \frac{1}{4} d_r \left(\nu+d_cd_r\right) \bbE \left\{ \frac{\beps_n^{'}  \left( \bSig_r^{-1} \dot{\bSig}_r^k \bSig_r^{-1}\otimes  \bSig_c^{-1} \right) \beps_n} {\nu+\delta_{\bX_n}(\btheta)} \right\}=\frac{1}{4} d_r \left(\nu+d_cd_r\right) \bbE \left\{ \tr\left( \frac{\beps_n^{'}  \left( \bSig_r^{-1} \dot{\bSig}_r^k \bSig_r^{-1}\otimes  \bSig_c^{-1} \right) \beps_n} {\nu+\delta_{\bX_n}(\btheta)} \right) \right\}\\
	&&= \frac{1}{4} d_r \left(\nu+d_cd_r\right) \tr\left\{ \left( \bSig_r^{-1} \dot{\bSig}_r^k \bSig_r^{-1}\otimes  \bSig_c^{-1} \right) \bbE \left( \frac{\beps_n \beps_n^{'}} {\nu+\delta_{\bX_n}(\btheta)} \right) \right\}=\frac{1}{4} d_cd_r \tr \left( \bSig_r^{-1} \dot{\bSig}_r^k \right), 
\end{IEEEeqnarray*}
and 
\begin{IEEEeqnarray*}{rCl}
	&&\hskip-1.5em\frac{1}{4} d_c \left(\nu+d_cd_r\right)\bbE \left\{ \frac{\beps_n^{'}  \left( \bSig_r^{-1} \otimes  \left( 
		\bSig_c^{-1} \dot{\bSig}_c^i \bSig_c^{-1} \right)  \right) \beps_n} {\nu+\delta_{\bX_n}(\btheta)} \right\} =\frac{1}{4} d_c \left(\nu+d_cd_r\right) \bbE \left\{ \tr\left(\frac{\beps_n^{'}  \left( \bSig_r^{-1} \otimes  \left( 
		\bSig_c^{-1} \dot{\bSig}_c^i \bSig_c^{-1} \right)  \right) \beps_n} {\nu+\delta_{\bX_n}(\btheta)} \right) \right\} \\
	&&\hskip1em =\frac{1}{4} d_c \left(\nu+d_cd_r\right) \tr\left\{ \left( \bSig_r^{-1} \otimes  \left( \bSig_c^{-1} \dot{\bSig}_c^i \bSig_c^{-1} \right)\right) \bbE \left( \frac{\beps_n \beps_n^{'}} {\nu+\delta_{\bX_n}(\btheta)} \right) \right\}=\frac{1}{4} d_cd_r \tr \left(\bSig_c^{-1} \dot{\bSig}_c^i\right)  
\end{IEEEeqnarray*}
respectively. Consequently, employing the information from \refe{eqn:gs} results in:
\begin{IEEEeqnarray*}{rCl}
	&&\hskip-2.5em\bbE\left\{\left(\nu+\delta_{\bX_n}(\btheta)\right)^{-2} \beps_n^{'} \left( \bSig_r^{-1} 
	\dot{\bSig}_r^k \bSig_r^{-1}  \otimes  \bSig_c^{-1} \right) \beps_n \beps_n^{'} \left(\bSig_r^{-1} \otimes\left(\bSig_c^{-1} \dot{\bSig}_c^i \bSig_c^{-1}\right)\right) \beps_n\right\} \\ 
	=&&\left(\nu+d_cd_r\right)^{-1}\left(\nu+d_cd_r+2\right)^{-1} 
	\left\{d_cd_r\tr\left(\bSig_r^{-1} \dot{\bSig}_r^k \right) \tr\left(\bSig_c^{-1} \dot{\bSig}_c^i \right) + 2\tr\left(\left(\bSig_r^{-1} \dot{\bSig}_r^k \right) \otimes \left(\bSig_c^{-1} \dot{\bSig}_c^i \right)\right)  \right\}.      
\end{IEEEeqnarray*}

(g) and (h) can be proven in a similar way. This completes the proof of \refthm{thm:Exp}. 
\end{proof}
	
	\begin{thm} \label{thm:info}
		For tBFA model in \refe{eqn:tbfa}, the Fisher information matrix $\bI_N(\bTheta)$ containing $N$ i.i.d. observations can be expressed as a block matrix \refe{eqn:inform},
		\begin{IEEEeqnarray}{rCl}\label{eqn:inform}
			\bI_N(\bTheta) =-\sum_{n=1}^{N} \bbE[\bH_n(\btheta\mid\cX)]= \begin{bmatrix}
				
				\bI_N^{\bmu\bmu}  &  \bo 	&\bo &\bo \\
				\bo 	&  \bI_N^{\btheta_c\btheta_c}	& \bI_N^{\btheta_c\btheta_r}& \bI_N^{\btheta_c\nu}\\
				\bo	&  \bI_N^{\btheta_r\btheta_c}&  \bI_N^{\btheta_r\btheta_r} & \bI_N^{\btheta_r\nu}\\
				\bo	&  \bI_N^{\nu\btheta_c}&  \bI_N^{\nu\btheta_r} & \bI_N^{\nu\nu}
			\end{bmatrix},
		\end{IEEEeqnarray}	
		Here, the specific expressions for each element of $\bI_N^{\btheta_c\btheta_c}$, $\bI_N^{\btheta_r\btheta_r}$ and $\bI_N^{\bmu\bmu}$, $\bI_N^{\nu \nu}$ are presented below:
		\begin{IEEEeqnarray*}{rCl}
			&& \bI_N^{\bmu\bmu} = N\frac{\nu+d_cd_r}{\nu+d_cd_r+2}\left(\bSig_r^{-1} \otimes \bSig_c^{-1}\right), \\
			&& \bI_N^{\nu\nu} = N \left\{\frac{1}{4}\bTG(\frac{\nu}{2})-\frac{1}{4}\bTG(\frac{\nu+d_cd_r}{2})-\frac{d_cd_r(\nu+d_cd_r+4)}{2\nu(\nu+d_cd_r)(\nu+d_cd_r+2)}\right\},\\
			&& \bI_N^{\theta_c^i\theta_c^j} = \frac{ Nd_r}{2(\nu+d_cd_r+2)}\left\{\left(\nu+d_cd_r \right) \tr \left(\bSig_c^{-1} \dot{\bSig}_c^i \bSig_c^{-1} \dot{\bSig}_c^j \right)-d_r\tr \left(\bSig_c^{-1} \dot{\bSig}_c^i\right)\tr \left(\bSig_c^{-1} \dot{\bSig}_c^j\right) \right\}.\\
			&& \bI_N^{\theta_c^i\theta_r^k} =\frac{ N}{2(\nu+d_cd_r+2)}\left\{\left(\nu+d_cd_r \right)\tr \left\{\left(\bSig_r^{-1} \dot{\bSig}_r^k\right) \otimes \left(\bSig_c^{-1} \dot{\bSig}_c^i\right)\right\}-d_cd_r\tr \left(\bSig_r^{-1} \dot{\bSig}_r^k\right)  \tr \left(\bSig_c^{-1} \dot{\bSig}_c^i\right)  \right\}.\\
			&& \bI_N^{\theta_c^i\nu} = \frac{-N d_r}{(\nu+d_cd_r)(\nu+d_cd_r+2)} \tr\left(\bSig_c^{-1} \dot{\bSig}_c^i \right),\\
			&& \bI_N^{\theta_r^k\theta_r^s} = \frac{ Nd_c}{2(\nu+d_cd_r+2)}\left\{\left(\nu+d_cd_r \right)\tr \left(\bSig_r^{-1} \dot{\bSig}_r^k \bSig_r^{-1} \dot{\bSig}_r^s \right)-d_c\tr \left(\bSig_r^{-1} \dot{\bSig}_r^k\right)\tr \left(\bSig_r^{-1} \dot{\bSig}_r^s\right) \right\}.\\
			&& \bI_N^{\theta_r^k\nu} = \frac{-N d_c}{(\nu+d_cd_r)(\nu+d_cd_r+2)} \tr\left(\bSig_r^{-1} \dot{\bSig}_r^k \right),
		\end{IEEEeqnarray*}
		where $\dot{\bSig}_c^i=\partial \bSig_c/\partial \theta_c^i=\bC(\partial \bC'/\partial \theta_c^i )+ (\partial \bC/\partial \theta_c^i )\bC'$ if $\theta_c^i \in \bC$, and then $\dot{\bSig}_c^i=\partial \bPsi_c / \partial \theta_c^i$ if $\theta_c^i \in \bPsi_c$. $\theta_c^i$ denotes the $i$-th element of $\btheta_c$, $i \in \{1,\dots,d_c(d_c+1)/2\}$. Similarly, $\dot{\bSig}_r^k=\partial \bSig_r/\partial \theta_r^k=\bR(\partial \bR'/\partial \theta_r^k )+ (\partial \bR/\partial \theta_r^k )\bR'$ if $\theta_r^k \in \bR$, and $\dot{\bSig}_r^i=\partial \bPsi_r / \partial \theta_r^k$ if $\theta_r^k \in \bPsi_r$,  $\theta_r^k$ denotes the $k$-th element of $\btheta_r$, $k \in \{1,\dots,d_r(d_r+1)/2\}$. Note that $\partial \bC / \partial \theta_c^i, \partial \bPsi_c/ \partial \theta_c^i,\partial \bR / \partial \theta_r^k$, and $\partial \bPsi_r/ \partial \theta_r^k$ are 0-1 matrices, i.e., each element of the matrix is 0 or 1.
	\end{thm}
\begin{proof}
	\noindent
{\bf Score vector:} The score vector $\bs\left(\btheta \mid \bx_n\right)$ is the vector of the first derivatives of $\ln f \left(\bx_n\right)$ w.r.t. $\btheta=(\bmu, \btheta_c, \btheta_r, \nu)$. It comprises the following entries:
\begin{IEEEeqnarray*}{rCl}
	\bs^\mu &&=\sum_{n=1}^{N}\frac{\left(\nu+d_cd_r\right)}{\left(\nu+\delta_{\bX_n}(\btheta)\right)} \left(\bSig_r^{-1}  \otimes \bSig_c^{-1}\right) \beps_n^{'}, \\
	\bs^{\theta_c^i} &&= -\frac{1}{2}\sum_{n=1}^{N} \left[d_r \tr \left(\bSig_c^{-1}\dot{\bSig}_c^i\right) - \left(\nu+d_cd_r\right)\tr \left\{ \frac{\bSig_c^{-1} \left(\bX_n-\bW\right)\bSig_r^{-1}\left(\bX_n-\bW\right)^{'}\bSig_c^{-1}\dot{\bSig}_c^i}{\nu+\delta_{\bX_n}(\btheta)} \right\} \right] , \\
	\bs^{\theta_r^k} &&= -\frac{1}{2}\sum_{n=1}^{N} \left[d_c \tr\left(\bSig_r^{-1}\dot{\bSig}_r^k \right) - \left(\nu+d_cd_r\right)\tr \left\{ \frac{ \bSig_r^{-1}\left(\bX_n-\bW \right)^{'}\bSig_c^{-1}\left(\bX_n-\bW\right)\bSig_r^{-1}\dot{\bSig}_r^k}{\nu+\delta_{\bX_n}(\btheta)} \right\} \right] ,\\
	\bs^{\nu}&&=\frac{1}{2}\sum_{n=1}^{N}\left[ \psi\left(\frac{\nu+d_cd_r}{2}\right)-\psi\left(\frac{\nu}{2}\right)+ \ln(\nu)+1-\ln \left(\nu+\delta_{\bX_n}(\btheta)\right)-\frac{\nu+d_cd_r}{\nu+\delta_{\bX_n}(\btheta)}
	\right].
\end{IEEEeqnarray*}

\noindent
{\bf The elements of the Hessian matrix are:}
\begin{IEEEeqnarray*}{rCl}
	\bH^{\bmu\bmu} 
	&&= \sum_{n=1}^{N} \left[ \frac{2\left(\nu+d_cd_r\right)}{\left( \nu+\delta_{\bX_n}(\btheta)\right)^2}\left(\bSig_r^{-1} \otimes \bSig_c^{-1} \right)\beps_n\beps_n^{'}\left(\bSig_r^{-1} \otimes \bSig_c^{-1} \right)-\frac{\nu+d_cd_r}{\nu+\delta_{\bX_n}(\btheta)}\left(\bSig_r^{-1} \otimes \bSig_c^{-1} \right) \right],\\
	\bH^{\bmu\theta_c^i} 
	&&= \sum_{n=1}^{N} \left[ \frac{\nu+d_cd_r}{\left( \nu+\delta_{\bX_n}(\btheta)\right)^2}\tr\left\{\bSig_c^{-1}\dot{\bSig}_c^i\bSig_c^{-1}   \left(\bX_n-\bW \right)\bSig_r^{-1} \left(\bX_n-\bW \right)^{'} \right \} \left(\bSig_r^{-1} \otimes \bSig_c^{-1}  \right)\beps_n\right. \\
	&&\quad \left.-\frac{\nu+d_cd_r}{\nu+\delta_{\bX_n}(\btheta)}\left\{\bSig_r^{-1} \otimes \left(\bSig_c^{-1}\dot{\bSig}_c^i\bSig_c^{-1} \right)\right\}\beps_n \right],\\ 
	\bH^{\bmu\theta_r^k} 
	&&= \sum_{n=1}^{N} \left[ \frac{\nu+d_cd_r}{\left( \nu+\delta_{\bX_n}(\btheta)\right)^2}\tr\left\{\bSig_r^{-1}\dot{\bSig}_r^k\bSig_r^{-1}   \left(\bX_n-\bW \right)^{'}\bSig_c^{-1} \left(\bX_n-\bW \right) \right \} \left(\bSig_r^{-1} \otimes \bSig_c^{-1}  \right)\beps_n\right. \\
	&&\quad \left.-\frac{\nu+d_cd_r}{\nu+\delta_{\bX_n}(\btheta)}\left\{\left(\bSig_r^{-1}\dot{\bSig}_r^k\bSig_r^{-1} \right) \otimes \bSig_c^{-1} \right\}\beps_n \right],\\
	\bH^{\bmu\nu} 
	&&= \sum_{n=1}^{N} \frac{\delta_{\bX_n}(\btheta)-d_cd_r}{\left( \nu+\delta_{\bX_n}(\btheta)\right)^2}\left(\bSig_r^{-1} \otimes \bSig_c^{-1}  \right)\beps_n,\\
	\bH^{\nu\nu} 
	&&= -\frac{1}{2}\sum_{n=1}^{N}\left[\frac{1}{\nu+\delta_{\bX_n}(\btheta)}+ \frac{\delta_{\bX_n}(\btheta)-d_cd_r}{(\nu+\delta_{\bX_n}(\btheta))^2} \right]+\frac{N}{4}\left[ \bTG(\frac{\nu+d_cd_r}{2})-\frac{1}{4}\bTG(\frac{\nu}{2})+\frac{2}{\nu} \right],\\
	\bH^{\theta_c^i \theta_c^j} 
	&&=-\frac{1}{2} \sum_{n=1}^{N} \left[ d_r \tr\left(\bSig_c^{-1}\ddot{\bSig}_c^{ij}-\bSig_c^{-1} \dot{\bSig}_c^j \bSig_c^{-1} \dot{\bSig}_c^i\right) \right.\\
	&&\quad +\left.\frac{\nu+d_cd_r}{\left( \nu+\delta_{\bX_n}(\btheta)\right)}\tr \left \{2\bSig_c^{-1} \dot{\bSig}_c^j \bSig_c^{-1} \dot{\bSig}_c^i \bSig_c^{-1}\left(\bX_n-\bW\right)\bSig_r^{-1}\left(\bX_n-\bW\right)' 
	-\bSig_c^{-1} \ddot{\bSig}_c^{ij} \bSig_c^{-1} \left(\bX_n-\bW\right)\bSig_r^{-1}\left(\bX_n-\bW\right)'\right\} \right.\\
	&&\quad \left.- \frac{\nu+d_cd_r}{\left( \nu+\delta_{\bX_n}(\btheta)\right)^2}\tr\left\{\bSig_c^{-1} \left(\bX_n-\bW\right)\bSig_r^{-1}\left(\bX_n-\bW\right)' \bSig_c^{-1}\dot{\bSig}_c^{j}\right\}  
	\tr\left\{\bSig_c^{-1} \left(\bX_n-\bW\right)\bSig_r^{-1}\left(\bX_n-\bW\right)'\bSig_c^{-1}\dot{\bSig}_c^{i}\right\} \right],\\
	\bH^{\theta_c^{i} \theta_r^{k}}
	&&= -\frac{1}{2}\sum_{n=1}^{N}\left[ \frac{\nu+d_cd_r}{ \nu+\delta_{\bX_n}(\btheta)}\tr \left\{ \bSig_c^{-1} \dot{\bSig}_c^i \bSig_c^{-1} \left(\bX_n-\bW\right) \bSig_r^{-1} \dot{\bSig}_r^k \bSig_r^{-1} \left(\bX_n-\bW\right)' \right\}\right. \\
	&&\quad \left.-\frac{\nu+d_cd_r}{\left( \nu+\delta_{\bX_n}(\btheta)\right)^2}\tr \left\{ \bSig_c^{-1} \dot{\bSig}_c^i \bSig_c^{-1} \left(\bX_n-\bW\right) \bSig_r^{-1}\left(\bX_n-\bW\right)' \right\}  
	\tr \left\{ \bSig_r^{-1} \dot{\bSig}_r^k \bSig_r^{-1} \left(\bX_n-\bW\right)' \bSig_c^{-1}\left(\bX_n-\bW\right)\right\} \right],
\end{IEEEeqnarray*}

\begin{IEEEeqnarray*}{rCl}
	\bH^{\theta_c^{i} \nu}
	&&= \frac{1}{2}\sum_{n=1}^{N} \frac{\delta_{\bX_n}(\btheta)-d_cd_r}{ \left(\nu+\delta_{\bX_n}(\btheta)\right)^2}\tr \left\{ \bSig_c^{-1} \left(\bX_n-\bW\right) \bSig_r^{-1}\left(\bX_n-\bW\right)' \bSig_c^{-1} \dot{\bSig}_c^i \right\},\\ 
	\bH^{\theta_r^k \theta_r^s} 
	&&=-\frac{1}{2} \sum_{n=1}^{N} \left [d_c \tr\left(\bSig_r^{-1}\ddot{\bSig}_r^{ks}-\bSig_r^{-1} \dot{\bSig}_r^s \bSig_r^{-1} \dot{\bSig}_r^k\right)\right.\\
	&& \quad +\left.\frac{\nu+d_cd_r}{\left( \nu+\delta_{\bX_n}(\btheta)\right)}\tr \left \{2\bSig_r^{-1} \dot{\bSig}_r^s \bSig_r^{-1} \dot{\bSig}_r^k \bSig_r^{-1}\left(\bX_n-\bW\right)' \bSig_c^{-1}\left(\bX_n-\bW\right)
	-\bSig_r^{-1} \ddot{\bSig}_r^{ks} \bSig_r^{-1} \left(\bX_n-\bW\right)'\bSig_c^{-1}\left(\bX_n-\bW\right)\right\} \right.\\
	&&\quad \left.- \frac{\nu+d_cd_r}{\left( \nu+\delta_{\bX_n}(\btheta)\right)^2}\tr\left\{\bSig_r^{-1} \left(\bX_n-\bW\right)'\bSig_c^{-1}\left(\bX_n-\bW\right)\bSig_r^{-1}\dot{\bSig}_r^{s}\right\} 
	\tr\left\{\bSig_r^{-1} \left(\bX_n-\bW\right)'\bSig_c^{-1}\left(\bX_n-\bW\right)\bSig_r^{-1}\dot{\bSig}_r^{k}\right\}\right],\\
	\bH^{\theta_r^{k} \nu}
	&&= \frac{1}{2}\sum_{n=1}^{N} \frac{\delta_{\bX_n}(\btheta)-d_cd_r}{ \left(\nu+\delta_{\bX_n}(\btheta)\right)^2}\tr \left\{ \bSig_r^{-1} \left(\bX_n-\bW\right)' \bSig_c^{-1}\left(\bX_n-\bW\right) \bSig_r^{-1} \dot{\bSig}_r^k \right\}, \\
	\bH^{\nu \nu} && = \frac{1}{2}\sum_{n=1}^{N} \left[\frac{1}{2}\bTG(\frac{\nu+d_cd_r}{2}) - \frac{1}{2}\bTG(\frac{\nu}{2})+\frac{1}{\nu}-\frac{1}{\nu+\delta_{\bX_n}(\btheta)}-\frac{\delta_{\bX_n}(\btheta)-d_cd_r}{(\nu+\delta_{\bX_n}(\btheta))^2}\right],
\end{IEEEeqnarray*}
where $\ddot{\bSig}_c^{ij} = \partial^2 \bSig_c / \partial \theta_c^i \partial \theta_c^j$ for each $i,j \in \{1,\dots,d_c(d_c+1)/2\}$, $\ddot{\bSig}_r^{ks} = \partial^2 \bSig_r / \partial \theta_r^k \partial \theta_r^s$ for each $k,s \in \{1,\dots,d_r(d_r+1)/2\}$.    
By \refthm{thm:Exp}, it can be shown that $\bbE \left[\bH^{\bmu\theta_c}\right] =\bo$, $\bbE \left[\bH^{\bmu\theta_r}\right] =\bo$ and $\bbE \left[\bH^{\bmu\nu}\right] =\bo$, and 
\begin{IEEEeqnarray*}{rCl}
	\bI_N^{\theta_c^i\theta_r^k} =&&-\bbE\left[\bH^{\theta_c^{i} \theta_r^{k}}\right] = \frac{1}{2}\sum_{n=1}^{N} \bbE  \left[\frac{\nu+d_cd_r}{ \nu+\delta_{\bX_n}(\btheta)}\tr \left\{ \bSig_c^{-1} \dot{\bSig}_c^i \bSig_c^{-1} \left(\bX_n-\bW\right) \bSig_r^{-1} \dot{\bSig}_r^k \bSig_r^{-1} \left(\bX_n-\bW\right)' \right\}\right. \\
	&&\left.-\frac{\nu+d_cd_r}{\left( \nu+\delta_{\bX_n}(\btheta)\right)^2}\tr \left\{ \bSig_c^{-1} \dot{\bSig}_c^i \bSig_c^{-1} \left(\bX_n-\bW\right) \bSig_r^{-1}\left(\bX_n-\bW\right)' \right\} \tr \left\{ \bSig_r^{-1} \dot{\bSig}_r^k \bSig_r^{-1} \left(\bX_n-\bW\right)' \bSig_c^{-1}\left(\bX_n-\bW\right)\right\} \right] \\
	=&& \frac{1}{2}\sum_{n=1}^{N} \left\{ \left(\nu+d_cd_r\right) \bbE  \left[\frac{\beps_n^{'}\left\{ \left( \bSig_r^{-1} \dot{\bSig}_r^k \bSig_r^{-1}\right) \otimes \left(\bSig_c^{-1} \dot{\bSig}_c^i \bSig_c^{-1} \right)\right\}\beps_n}{ \nu+\delta_{\bX_n}(\btheta)} \right.\right] \\
	&&\left.-\left(\nu+d_cd_r\right)\bbE \left[\frac{\beps_n^{'}\left(\bSig_r^{-1} \dot{\bSig}_r^k \bSig_r^{-1}\otimes \bSig_c^{-1}\right)\beps_n \beps_n^{'} \left(\bSig_r^{-1} \otimes \left(\bSig_c^{-1} \dot{\bSig}_c^i \bSig_c^{-1}\right)\right)\beps_n}{\left( \nu+\delta_{\bX_n}(\btheta)\right)^2}\right] \right\} \\
	=&& \frac{1}{2}\sum_{n=1}^{N} \left\{ \left(\nu+d_cd_r\right)\left\{\left( \bSig_r^{-1} \dot{\bSig}_r^k \bSig_r^{-1}\right) \otimes \left(\bSig_c^{-1} \dot{\bSig}_c^i \bSig_c^{-1} \right)\right\}  \bbE  \left[\frac{\beps_n \beps_n^{'}}{ \nu+\delta_{\bX_n}(\btheta)} \right.\right] \\
	&&\left.-\left(\nu+d_cd_r\right)\bbE \left[\frac{\beps_n^{'}\left(\bSig_r^{-1} \dot{\bSig}_r^k \bSig_r^{-1}\otimes \bSig_c^{-1}\right)\beps_n \beps_n^{'} \left(\bSig_r^{-1} \otimes \left(\bSig_c^{-1} \dot{\bSig}_c^i \bSig_c^{-1}\right)\right)\beps_n}{\left( \nu+\delta_{\bX_n}(\btheta)\right)^2}\right] \right\}\\
	=&& \frac{1}{2}\sum_{n=1}^{N} \left[\left(\nu+d_cd_r\right)\left\{\left( \bSig_r^{-1} \dot{\bSig}_r^k \right) \otimes \left(\bSig_c^{-1} \dot{\bSig}_c^i \right)\right\}  \right. \\
	&&\left.-\left(\nu+d_cd_r\right)\left(\nu+d_cd_r\right)^{-1}\left(\nu+d_cd_r+2\right)^{-1} \left\{d_cd_r\tr\left(\bSig_r^{-1} \dot{\bSig}_r^k \right) \tr\left(\bSig_c^{-1} \dot{\bSig}_c^i \right) 
	+2\tr\left(\left(\bSig_r^{-1} \dot{\bSig}_r^k \right) \otimes\left(\bSig_c^{-1} \dot{\bSig}_c^i \right)\right)  \right\}  \right]\\
	=&&\frac{ N}{2(\nu+d_cd_r+2)}\left[\left(\nu+d_cd_r \right)\tr \left\{\left(\bSig_r^{-1} \dot{\bSig}_r^k\right) \otimes  \left(\bSig_c^{-1} \dot{\bSig}_c^i\right) \right\}-d_cd_r\tr \left(\bSig_r^{-1} \dot{\bSig}_r^k\right)  \tr \left(\bSig_c^{-1} \dot{\bSig}_c^i\right) \right].
\end{IEEEeqnarray*}
Other cases can be demonstrated using a similar method.
\end{proof}
	
	\section{Experiments}\label{sec:expr}
	In this section, a series of experiments are conducted on simulated and real data to understand our proposed \emph{t}BFA and compare the performance of  \emph{t}BFA with three relevant methods: BFA \citep{zhao2023-BFA}, \emph{t}FA \citep{zhang2014robust} and FA. All computations are performed by Matlab 9.13 on a desktop computer with Intel Core i7-8700K 3.70GHz CPU and 64G RAM.
	
	\subsection{Notes on implementation}
	\subsubsection{Initialization and stopping rules}
	For all methods, we set $\eta = 0.005$. Unless otherwise stated, all iterative algorithms are initialized randomly. Iterations are stopped when the relative change in the log-likelihood function ($|1-\cL^{(t)}/\cL^{(t+1)}|$) is smaller than a threshold $tol$ (set to $10^{-8}$ in our experiments) or when the number of iterations exceeded a certain value $t_{max}$ (set to $1000$ in our experiments). 
	
	\subsubsection{Model selection}
	In practice, the number of factors $(q_c, q_r)$ in \tbfa is often unknown. To determine $(q_c, q_r)$, we adopt the Bayesian Information Criterion (BIC). BIC is well known for its theoretical consistency \citep{Schwarz1978}. The BIC is defined as follows:
	\begin{equation}\label{eqn:BIC}
		\cL^{*}\left\{(q_c,q_r),\bwTheta(q_c,q_r)\right\} = -2\cL\left\{\bwTheta(q_c,q_r)\mid\cX\right\}+\cD \ln N,
	\end{equation} 
	where $\bwTheta(q_c,q_r)$ represents the ML estimate of parameters $\bTheta$ in the $(q_c,q_r)$-factor model, $\cL(\bwTheta\mid\cX)$ denotes the value of the log-likelihood function \refe{eqn:tbfa.like} evaluated at $\bwTheta(q_c,q_r)$, and $\cD$ is the number of free parameters.
	
	
	By BIC, $(q_c, q_r)$ is determined by the mimimizer of $\cL^{*}$ in the candidate set $\mathcal{I}=\{(q_c,q_r) \mid q_{c,min}\leq q_c\leq q_{c,max},q_{r,min}\leq q_r\leq q_{r,max}\}$. Suppose that the true model is within the candidate set, due to the consistency of the BIC, as the sample size $N\rightarrow \infty$, the true model will be selected with probability approaching one.
	
	\subsection{Synthetic data}
	\subsubsection{Convergence of different algorithms}\label{sec:cov}
	In this experiment, we use three simulated datasets (i.e., Data 1, Data2 and Data3) to compare the convergence performance of four ML estimation algorithms for \emph{t}BFA: ECME, PX-ECME, AECM, and PX-AECM. These datasets are generated by BFA, namely $\bX\sim \cN_{d_c,d_r}(\bW,\bC\bC'+\bPsi_c,\bR\bR'+\bPsi_r)$. Data1 and Data2 are two low-dimensional datasets and Data3 is a high-dimensional one. The details about the data settings are given below.
	
	(i) Data1: Low-dimensional data with ordinary noise. The sample size $N=500$. The data dimensions $d_c = d_r = 10$, and the numbers of factors $q_c = q_r = 3$. Let $\bU_c$ (resp. $\bU_r$) is a $d_c\times q_c$ (resp. $d_r\times q_r$) randomly generated column-orthogonal matrix and $\bL_c$ (resp. $\bL_r$) is a $q_c\times q_c$ ($q_r\times q_r$) diagonal matrix. The parameter settings in BFA model are given by 
	\begin{IEEEeqnarray*}{rCl}
		&&\bW = \bo,\quad\bL_c = \diag\{\sqrt{5},\sqrt{4.5},\sqrt{4}\}, \quad\bL_r = \diag\{\sqrt{10},\sqrt{9},\sqrt{8}\}\\
		&&\bC = \bU_c\bL_c, \,\,\bR = \bU_r\bL_r,\,\,\bPsi_c = \diag\{\mbox{linspace}(0.5,1,10)\}, \,\,\bPsi_r = \diag\{\mbox{linspace}(1,2,10)\},
	\end{IEEEeqnarray*}
	where \mbox{linspace}$(a, b, n)$ represents a vector of $n$ evenly spaced points from $a$ to $b$, i.e., $(a,a+(b-a)/(n-1),\dots,b)'$. 
	
	(ii) Data2: Low-dimensional data with low noise. As EM algorithms tend to be inefficient in low-noise linear models \citep{Kaar-slowem,zhao2008-efa}, we examine the performance of four algorithms under low-noise conditions. For Data2, $\bPsi_c = \diag\{\mbox{linspace}(0.05,0.1,10)\}$, $\bPsi_r = \diag\{\mbox{linspace}(0.1,0.2,10)\}$ are used to simulate low noise, with other parameter settings identical to Data1.
	
	(iii) Data3: High-dimensional data. Due to the lower computational complexity at a single iteration of the AECM algorithm compared to the ECME algorithm, we expect the AECM algorithm to consume less CPU time on high-dimensional data. To explore this, we use high-dimensional Data3 to compare the convergence performance of the four algorithms. The data dimensions are $d_c = 2000, d_r = 10$, the numbers of factors are also $q_c = q_r = 3$, the sample size is $N=100$, and other parameter settings are similar as Data1.
	
	\begin{figure*}[tb]
		\centering \scalebox{0.65}[0.53]{\includegraphics*{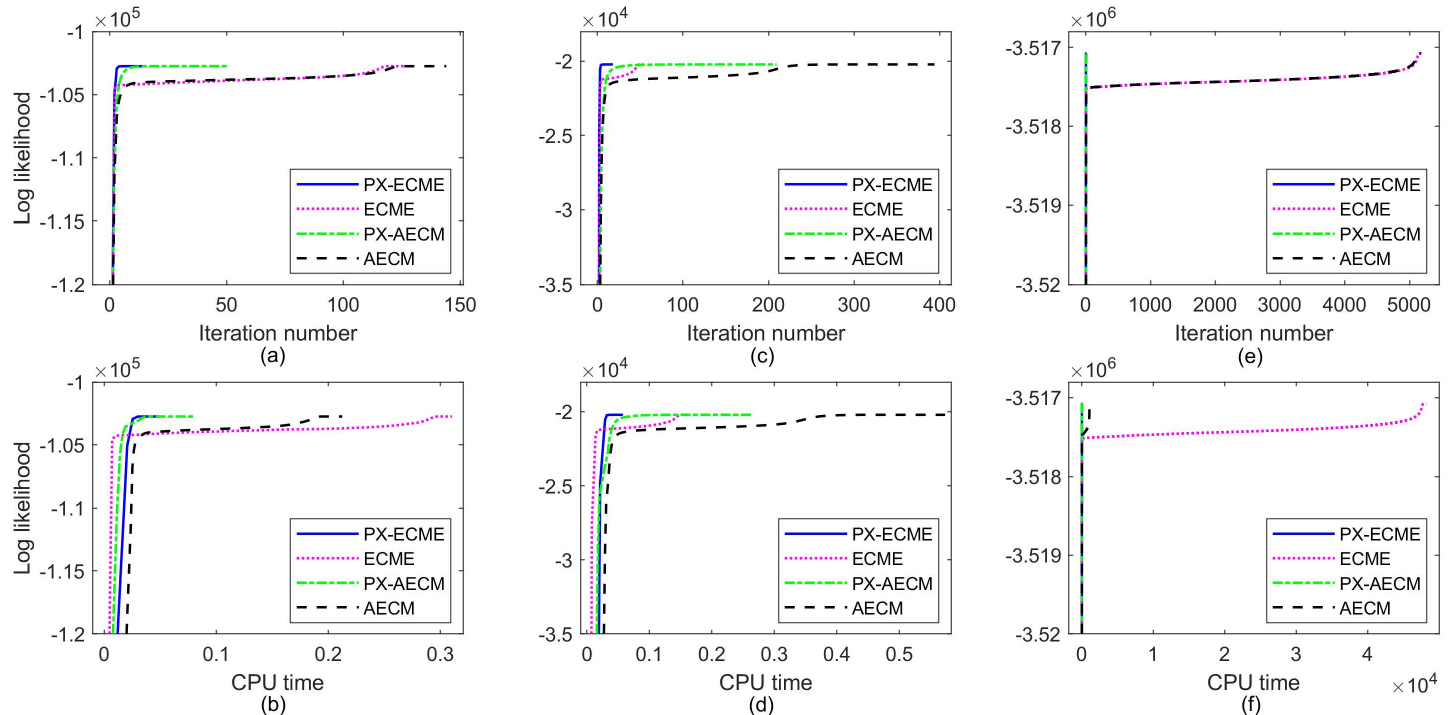}}
		\caption{Changes in log-likelihood $\cL$ for PX-ECME (solid), ECME (dotted), PX-AECM (dashdot), and AECM (dashed) algorithms with (a) number of iterations on Data1; (b) CPU time on Data1; (c) number of iterations on Data2; (d) CPU time on Data2; (e) number of iterations on Data3; (f) CPU time on Data3.} \label{fig:cov}
	\end{figure*}
	
	All four algorithms use the true factor dimensions $(3,3)$ to fit these three datasets. For demonstration purpose, we set $tol=10^{-9}$ for Data3. \reff{fig:cov} displays the evolution of the log-likelihood versus the number of iterations and used CPU time for the four algorithms on these three datasets. The main observations are as follows.
	\begin{enumerate}[(i)]
		\item For Data1, both PX-ECME and ECME converge faster compared to their counterparts PX-AECM and AECM; PX algorithms converge faster than their non-PX counterparts.
		\item For Data2, the conclusions are roughly similar to those on Data1. However, AECM-type algorithms converge much slower on Data2 than on Data1, while ECME-type algorithms are insensitive to low noise.
		\item For Data3, AECM-type algorithms consume less time compared to their ECME-type counterparts; PX algorithms are more efficient than their non-PX counterparts.
	\end{enumerate}
	
	Overall, PX-ECME algorithm performs well on different datasets. Therefore, we will use PX-ECME algorithm for \emph{t}BFA in the following experiments.
	
	\subsubsection{Robustness}\label{sec:rbst}
	
	This experiment investigates the robustness of \emph{t}BFA, BFA, \emph{t}FA, and FA using simulated data containing outliers. Initially, we generate a dataset with $N = 1000$ normal observations from Data1 in \refs{sec:cov}. Subsequently, we add it with $Np/(1-p)$ abnormal observations so that the proportion of all outliers is $p$. Inspired by \cite{zhao2023-rfpca}, we considered three types of outliers under four different situations:
	\begin{enumerate}[(i)]
		\item Three types of outliers: 1) Factor Component (FC) outliers, showing outlyingness \emph{only} in the FC subspace; 2) Orthogonal Component (OC) outliers, showing outlyingness \emph{only} in the OC subspace; 3) FC+OC outliers, showing outlyingness in both the FC and OC subspaces.
		\item Four situations (Sit): two abnormal levels in two cases (symmetry and asymmetry). To be specific, the three types of outliers are generated from the following situations: Sit-I: $U(-100,100)$, Sit-II: $U(-10000,10000)$, Sit-III: $U(100,110)$, and Sit-IV: $U(10000,11000)$. In the symmetric cases, the outliers are symmetrical about the true population mean and thus their impacts on the estimate of the mean parameter may be relatively small.
	\end{enumerate}
	In total, 12 outlier scenarios are investigated. For demonstration, \reff{fig:outliers} presents a two-dimensional subspace plot that illustrates the three types of outliers in an asymmetric case. Further details on generating the three types of outliers are provided in \refap{sec:gout}.
	
	\begin{figure*}[htb]
		\centering \scalebox{0.65}[0.55]{\includegraphics*{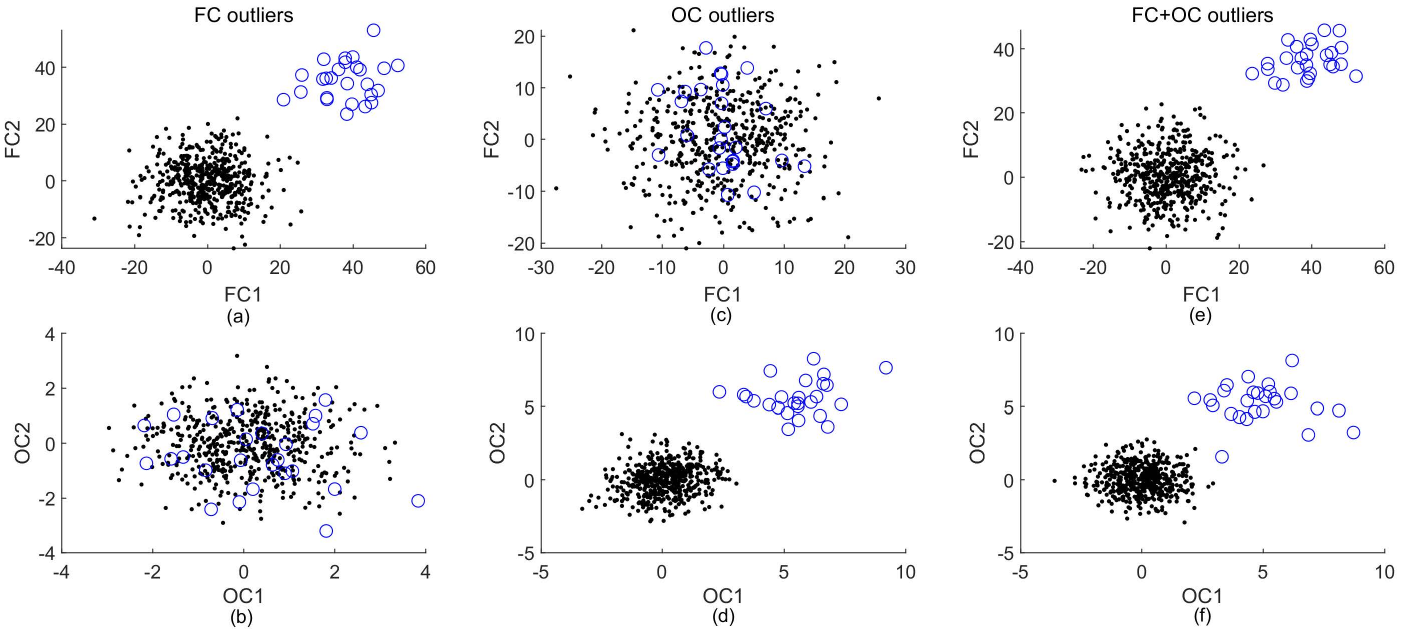}} 
		\caption{The two-dimensional subspace plot of the three kinds of outliers in an asymmetric case, where the observations and outliers are shown as black points and blue circles, respectively. First column: FC outliers in the (a) FC subspace, and (b) OC subspace; Second column: OC outliers in the (c) FC subspace, and (d) OC subspace; Third column: FC+OC outliers in the (e) FC subspace, and (f) OC subspace.}
		\label{fig:outliers}
	\end{figure*}
	
	To assess the accuracy of parameter estimates obtained by different methods, we report the relative differences between the true covariance matrix $\bSig=\bSig_r\otimes\bSig_c$ and the estimated $\bwSig$, denoted as $\Vert\bSig-\bwSig \Vert_F/\Vert\bSig\Vert_F$, where $\Vert\cdot\Vert_F$ represents the Frobenius norm. For FA and \emph{t}FA, $\bwSig=\bwA\bwA'+\bwPsi$. For BFA and \emph{t}BFA, $\bwSig=\bwSig_r\otimes\bwSig_c$, where $\bwSig_c=\bwC\bwC'+\bwPsi_c$ and $\bwSig_r=\bwR\bwR'+\bwPsi_r$. To reduce statistical variability, we report the average results over 50 repetitions.
	
	\begin{table*}[!t]
		\centering
		\caption{\label{tab:robust} Relative difference between the estimated and true covariance matrix for multiple $p$'s in six cases (i.e., three kinds of outliers (FC, OC, and FC+OC) in symmetric and asymmetric cases).}
		\begin{tabular}{cccrrrrrrrrrr}
			\toprule
			\multirow{2}{*}{Type} & \multirow{2}{*}{Method} &\multirow{2}{*}{\makecell[c]{Without\\outliers}}    & \multicolumn{10}{c}{Outliers with a proportion $p$}      \\ \cmidrule{4-13}
			&                         &          &0.5\%   &1\%   &2\%   &5\%   &9\%        &0.5\%   &1\%   &2\%   &5\%   &9\%  \\ \midrule
			&                         &          &\multicolumn{5}{c}{Sit-I: $U(-100,100)$} & \multicolumn{5}{c}{Sit-III: $U(100,110)$} \\ \cmidrule(r){4-8} \cmidrule{9-13}
			\multirow{4}{*}{FC}   &       \emph{t}BFA       &    0.0   &   0.0  & 0.1  & 0.1  & 0.4  & 0.8       &   0.1  &  0.1 & 0.2  & 1.2  & 14.5  \\
			&            BFA          &    0.0   &   4.7  & 8.7  & 16.4 & 40.0 & 69.5      &  13.5  & 25.5 & 48.2 &116.2 & 196.9  \\
			&       \emph{t}FA        &    0.2   &   0.3  & 0.3  &  0.4 & 1.4  & 21.8      &   0.5  & 55.3 &240.3 &774.2 &1442.9  \\
			&            FA           &    0.2   &  18.5  & 29.5 & 52.7 &127.1 & 219.9     & 113.8  &225.4 &442.0 &1099.1&1884.7  \\ \midrule 
			\multirow{4}{*}{OC}   &       \emph{t}BFA       &   0.0   &   0.0  & 0.0  & 0.1  & 0.1  & 0.1       &   0.0  &  0.0 & 0.1  & 0.1  & 0.5  \\
			&            BFA          &    0.0   &   0.9  & 1.4  & 2.3  & 5.5  & 9.8       &  11.9  & 22.6 & 40.8 & 83.7 & 125.3  \\
			&       \emph{t}FA        &    0.2   &   0.2  & 0.2  & 0.2  & 0.2  & 0.2       &   0.2  &  0.2 &  0.2 & 71.5 & 133.3  \\
			&            FA           &    0.2   &   2.3  & 3.3  & 4.6  & 8.2  & 12.3      &  15.3  & 30.3 & 59.5 &148.0 & 253.8  \\ \midrule
			\multirow{4}{*}{FC+OC}&       \emph{t}BFA       &   0.0   &   0.0  & 0.0  & 0.0  & 0.0  & 0.1       &   0.0  &  0.1 & 0.1  & 0.3  & 2.1  \\
			&            BFA          &    0.0   &  14.3  &26.2  &50.5  &128.1 &228.1      &  29.4  & 53.5 & 96.6 &226.5 & 379.7  \\
			&       \emph{t}FA        &    0.2   &   0.2  & 0.2  & 0.2  & 0.2  & 0.2       &   0.2  &  0.3 & 73.9 &253.7 & 491.8  \\
			&            FA           &    0.2   &  32.6  &48.9  &74.3  &150.8 &244.6      & 217.8  &431.2 &846.5 &2103.6&3609.4  \\ \bottomrule                       	
		\end{tabular}
	\end{table*}
	\reft{tab:robust} displays the results for Sit-I and Sit-III at varying proportions $p$. The results for Sit-II and Sit-IV, with higher abnormal levels, are similar to those in \reft{tab:robust}, and details are omitted here. Looking at \reft{tab:robust}, we have the following.
	\begin{enumerate}[(i)]
		\item \emph{t}BFA versus BFA and FA: \emph{t}BFA is more robust than BFA and FA, both of which are susceptible to outliers. The larger the value of $p$, the greater the distance.
		\item \emph{t}BFA versus \emph{t}FA: \emph{t}FA is not robust except for OC and FC+OC outliers in the symmetric case. Notably, for FC outliers in the asymmetric case, \emph{t}FA is only robust when the proportion of outliers is 0.5\%, but not when it exceeds 1\%. In contrast, \emph{t}BFA maintains robustness across all $p$ values in all six cases. This implies that \emph{t}BFA can handle datasets with a significantly higher proportion of outliers compared to \emph{t}FA.
	\end{enumerate}
	
	\emph{Remarks}: Our results obtained in this subsection are generally comparable with those in \cite{zhao2023-rfpca}. The main reason for the superior performance of \tbfa over \tfa is given below. It is known that the robustness of estimators can be assessed through breakdown points, and  the breakdown point of the multivariate \emph{t} distribution is upper-bounded by $1/(d+\nu)$ \citep{Dumbgen-bkd} and thus smaller than $1/d$. For our used data, $d=cr=100$ and thus $1/d=1\%$. This elucidates the complete failure of \tfa when $p\geq1\%$ for PC outliers in the asymmetric case. However, this bound does not hold for \tbfa, since \tbfa still performs well even when $p=9\%$ across all six cases. This suggests that the proposed \emph{t}BFA has a significantly higher bound than \emph{t}FA, making \emph{t}BFA more suitable for matrix data.

	\subsubsection{Accuracies of estimators}\label{sec:acc}
	This experiment investigates the finite sample performance of the ML estimators and the estimated standard errors. The experiment utilizes $5\times5$-dimensional matrix datasets generated by \tbfa \refe{eqn:tbfa}, with the following parameter settings:
	\begin{IEEEeqnarray*}{rCl}
		\begin{split}
			\bW&=\bo,\quad \nu=3,\quad 
			\bPsi_c = \diag\{0.1,0.2,0.3,0.4,0.5\},\quad      \bPsi_r = \diag\{0.2,0.3,0.4,0.5,0.6\},\\
			\bC'&=
			\begin{bmatrix}
				-1.03 & -0.78 & -1.35 & -1.05 & -2.10 \\
				3.47  & -4.39 & 6.99 & -3.44 & -2.83
			\end{bmatrix},\quad
			\bR' = \begin{bmatrix}
				-1.12 & -1.40 & -1.45 & -1.54 & -1.15 \\
				-2.15  & -5.44 & -4.71 & 6.28 & 2.04,
			\end{bmatrix}.
		\end{split}
	\end{IEEEeqnarray*}
	
	To study performance under finite sample conditions, datasets with different sample sizes $N$ are generated from the set $N \in \{100, 500, 5000\}$. We fit \emph{t}BFA with the true factor dimensions $q_c=q_r=2$ for each dataset. To reduce the effects of randomness, we perform 100 repetitions for each dataset. The following metrics are calculated.
	
	The performance of the ML estimators is evaluated using the root mean square error (RMSE):
	\begin{equation*}
		\hbox{RMSE}\left(\bwTheta_i\right)=\sqrt{\frac{1}{100} \sum_{h=1}^{100}\left(\bwTheta_i^{(h)}-\bTheta_i\right)^2},
	\end{equation*}
	where $\bwTheta^{(h)}$ represents the  ML estimators of the true parameters $\bTheta$ obtained in the $h$-th repetition.
	
	Additionally, to investigate the consistency of the estimated standard errors, the average values of information matrix-based standard errors (IMSE) in \refe{eqn:inform} and empirical standard deviations (ESTD) are calculated:
	\begin{equation*}
		\hbox{IMSE}\left(\bwTheta_i\right) = \frac{1}{100}\sum_{h=1}^{100}\text{SE}(\bwTheta_i^{(h)}), \quad
		\hbox{ESTD}\left(\bwTheta_i\right) = \sqrt{\frac{1}   {99}\left[\sum_{h=1}^{100}\left(\bwTheta_i^{(h)}\right)^2-\frac{1}{100}\left(\sum_{h=1}^{100}\bwTheta_i^{(h)}\right)^2\right]},
	\end{equation*}
	where $\text{SE}\left(\bwTheta_i^{(h)}\right)$ denotes the information matrix-based standard errors (IMSE) of $\bTheta_i$ in the $h$-th repetition in \refs{sec:stderr}. Note that the restrictions mentioned in \refs{sec:tbfa.identi} are imposed on \emph{t}BFA to ensure the uniqueness of the estimators.
	
	\begin{table}[htb]
		\centering
		\caption{\label{tab:rmse} RMSEs of the ML estimates.}
		\begin{tabular}{cccccccccccc}
			\toprule
			$\bC,\nu$         &  $N$ &$c_{21}$ &$c_{31}$ &$c_{41}$ &$c_{51}$ &$c_{12}$ &$c_{22}$ &$c_{32}$ &$c_{42}$ &$c_{52}$ &$\nu$ \\ \midrule
			\multirow{3}{*}{RMSE} & 100  & 0.122   & 0.110   & 0.111   & 0.148   & 0.211   & 0.241   & 0.392   & 0.192   & 0.170   & 0.548\\
			& 500  & 0.054   & 0.053   & 0.051   & 0.068   & 0.095   & 0.103   & 0.177   & 0.082   & 0.071   & 0.220\\
			& 5000 & 0.018   & 0.017   & 0.017   & 0.023   & 0.027   & 0.032   & 0.051   & 0.025   & 0.023   & 0.061\\ \midrule
			$ \bPsi_c,\bPsi_r$    & $N$  &$\psi_{c,2}$ &$\psi_{c,3}$ &$\psi_{c,4}$ &$\psi_{c,5}$ &$\psi_{r,1}$ &$\psi_{r,2}$ &$\psi_{r,3}$ &$\psi_{r,4}$ &$\psi_{r,5}$\\ \midrule
			\multirow{3}{*}{RMSE} & 100  & 0.029       & 0.157       & 0.048       & 0.072       & 0.037       & 0.118       & 0.122       & 0.114       & 0.081 \\
			& 500  & 0.012       & 0.067       & 0.022       & 0.033       & 0.014       & 0.044       & 0.047       & 0.048       & 0.035 \\
			& 5000 & 0.005       & 0.019       & 0.007       & 0.010       & 0.005       & 0.014       & 0.013       & 0.014       & 0.010 \\ \bottomrule
		\end{tabular}
	\end{table}
	
	\begin{table}[htb]
		\centering
		\caption{\label{tab:std} STDs and IMSEs of the ML estimates.}
		\begin{tabular}{cccccccccccc}
			\toprule
			$N$              & Measure &$c_{21}$ &$c_{31}$ &$c_{41}$ &$c_{51}$ &$c_{12}$ &$c_{22}$ &$c_{32}$ &$c_{42}$ &$c_{52}$ &$\nu$ \\ \midrule
			\multirow{2}{*}{100} &  ESTD   & 0.122   & 0.110   & 0.111   & 0.149   & 0.212   & 0.242   & 0.394   & 0.192   & 0.171   & 0.538    \\
			&  IMSE   & 0.124   & 0.113   & 0.118   & 0.159   & 0.199   & 0.229   & 0.369   & 0.184   & 0.175   & 0.444   \\
			\multirow{2}{*}{500} &  ESTD   & 0.054   & 0.054   & 0.051   & 0.068   & 0.095   & 0.102   & 0.176   & 0.081   & 0.071   & 0.218    \\
			&  IMSE   & 0.055   & 0.050   & 0.053   & 0.071   & 0.089   & 0.102   & 0.165   & 0.082   & 0.078   & 0.192   \\
			\multirow{2}{*}{5000} &  ESTD   & 0.018   & 0.016   & 0.017   & 0.023   & 0.027   & 0.033   & 0.052   & 0.025   & 0.023   & 0.062    \\
			&  IMSE   & 0.018   & 0.016   & 0.017   & 0.023   & 0.028   & 0.032   & 0.052   & 0.026   & 0.025   & 0.060   \\ \midrule
			$N$       & Measure &$\psi_{c,2}$ &$\psi_{c,3}$ &$\psi_{c,4}$ &$\psi_{c,5}$ &$\psi_{r,1}$ &$\psi_{r,2}$ &$\psi_{r,3}$ &$\psi_{r,4}$ &$\psi_{r,5}$\\ \midrule
			\multirow{2}{*}{100} &  ESTD   & 0.029       & 0.156       & 0.048       & 0.072       & 0.037       & 0.118       & 0.120       & 0.113       & 0.082 \\
			&  IMSE   & 0.029       & 0.154       & 0.049       & 0.074       & 0.035       & 0.110       & 0.105       & 0.109       & 0.083  \\
			\multirow{2}{*}{500} &  ESTD   & 0.012       & 0.067       & 0.022       & 0.033       & 0.014       & 0.044       & 0.047       & 0.048       & 0.035 \\
			&  IMSE   & 0.013       & 0.069       & 0.022       & 0.033       & 0.016       & 0.048       & 0.046       & 0.048       & 0.037  \\
			\multirow{2}{*}{5000} &  ESTD   & 0.005       & 0.019       & 0.007       & 0.010       & 0.005       & 0.014       & 0.013       & 0.014       & 0.010 \\
			&  IMSE   & 0.004       & 0.022       & 0.007       & 0.011       & 0.005       & 0.015       & 0.014       & 0.015       & 0.012  \\\bottomrule
		\end{tabular}
	\end{table}
	
	\reft{tab:rmse} provides the root mean squared error (RMSE) values of the ML estimates for four parameters $\bC, \nu, \bPsi_c, \bPsi_r$ under different sample sizes. The results for $\bR$ are similar to those for $\bC$ and are omitted here. As shown in \reft{tab:rmse}, with the increase in sample size, the RMSE values of all model parameters gradually decrease, indicating more accurate model estimation and demonstrating the consistency of the ML estimator.
	
	\reft{tab:std} presents the estimated standard errors of the ML estimates for four parameters. From \reft{tab:std}, it can be observed that both ESTDs and IMSEs decrease as sample size $N$ increases. Importantly, the two values approach each other, indicating the accuracy and stability of the theoretical standard errors calculated based on the information matrix.
	

	\subsection{Real data}\label{sec:real}
	In this section, we use real data to investigate the proposed \emph{t}BFA, including the performance of BIC-based model selection, the interpretation of row and column common factors, and the prediction of matrix factor scores.
	
	\subsubsection{Athletic records data}
	Although factor analysis is widely used, the results in the literature mostly focus on vector data and there is relative paucity of the analysis on matrix data directly. For better comparison, in this experiment, vector data is artificially augmented to matrix data so that we can utilize the known factor analysis results on vector data as a reference. 
	
	
	We use a vector-valued dataset from Exercise 12.8 in \cite{Haerdle2019-fa}, denoted by Data-V. The dataset is composed of men’s athletic records of 8 track events from 55 countries/regions in the 1984 Olympic Games. That is, Data-V consists of 55 observations and 8 variables. \reft{tab:data.sport} provides a summary of variable information, including sample means, variances, kurtosis, and Jarque-Bera tests \citep{jarque1980efficient}. As indicated by the descriptive statistics in the table, all variables exhibit high kurtosis. The Jarque-Bera test $p$-values suggest that none of the variables follow a normal distribution. Therefore, modeling with heavy-tailed distributions would be appropriate.
	
	\begin{table}[htb]
		\centering
		\caption{\label{tab:data.sport} An overview of Men’s track records dataset.}
		\begin{tabular}{clcccc}
			\toprule
			\multirow{2}{*}{Variable}  &\multirow{2}{*}{Description}  &\multirow{2}{*}{\makecell[c]{Sample\\Mean}}  &\multirow{2}{*}{\makecell[c]{Sample\\Variance}}  &\multirow{2}{*}{\makecell[c]{Sample\\kurtosis}}   &\multirow{2}{*}{\makecell[c]{$p$-value of the \\Jarque-Bera test}} \\ 
			~       &       ~            &   ~     &   ~       &  ~         &  ~ \\ \midrule
			$\bx_1$     &100 meters (s)      & 10.47   & 0.12      & 11.50      & < 0.001 \\
			$\bx_2$     &200 meters (s)      & 20.94   & 0.42      & 4.64       & 0.009 \\
			$\bx_3$     &400 meters (s)      & 46.44   & 2.12      & 8.43       & < 0.001 \\
			$\bx_4$     &800 meters (min)    & 1.79    & 0.00      & 6.79       & < 0.001 \\
			$\bx_5$     &1500 meters (min)   & 3.70    & 0.02      & 6.24       & < 0.001 \\
			$\bx_6$     &5000 meters (min)   & 13.85   & 0.64      & 5.40       & < 0.001 \\
			$\bx_7$     &10000 meters (min)  & 28.99   & 3.27      & 5.64       & < 0.001 \\
			$\bx_8$     &marathon (min)      & 136.62  & 85.14     & 4.19       & 0.003 \\ \bottomrule
		\end{tabular}
	\end{table}
	
	To select the factor dimension $q$, we use BIC for FA and \emph{t}FA on the original Data-V. The candidate model set is defined as $\mathcal{I}=\{q \mid 1 \leq q \leq 4\}$, encompassing 4 models. \reff{fig:sport.bic} (a) shows the evolutions of the BIC values of all candidates with varying $(q)$ on Data-V for both FA and \tfa. \reft{tab:sport.bic} summarizes the optimal results. It can be observed that both FA and \emph{t}FA favor a 2-factor model. However, it is noticeable that the BIC values for \emph{t}FA models are consistently smaller than those for FA models, suggesting that the 2-factor \emph{t}FA model provides the best fit.
	
	\begin{figure*}[htb]
		\centering \scalebox{0.8}[0.6]{\includegraphics*{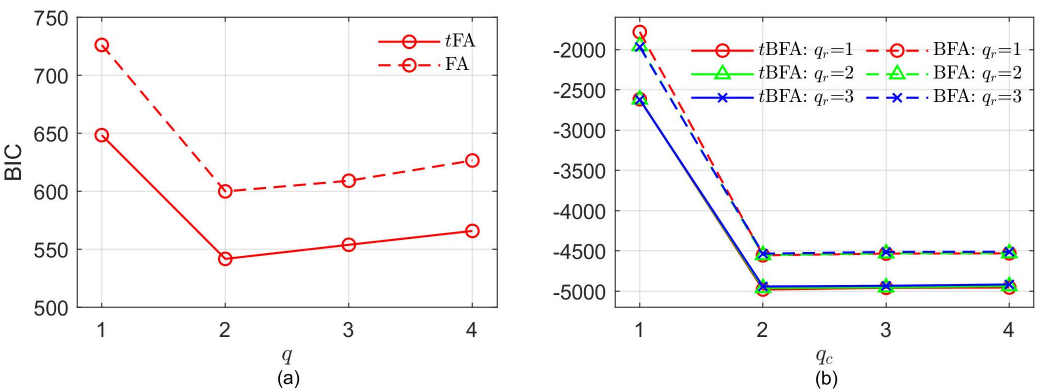}} 
		\caption{(a) BIC values versus various $q$ for FA and \tfa; (b) BIC values versus various $(q_c,q_r)$ for BFA and \tbfa.}
		\label{fig:sport.bic}
	\end{figure*}
	
	Next, we augment Data-V and compare the performance of \emph{t}BFA and BFA on the augmented matrix data. The original Data-V can be considered as a special $8 \times 1 \times 55$ matrix-valued dataset, i.e., $d_c = 8, d_r = 1$. Based on the results of FA, with row factor number $q_c = 2$, column factor number $q_r = 1$, and the estimated factor loading matrix denoted as $\hat{\bA}$, combined with the estimated degrees of freedom $\hat{\nu}=4.28$ from \emph{t}FA, we set $\bC=\hat{\bA}$, $\nu=4.28$, and randomly generate a $d_r \times 1$ matrix $\bR$. By \refe{eqn:tbfa.r}, we can augment Data-V from $d_r=1$ to any $d_r$ ($d_r>1$). In this experiment, we set $d_r=7$, resulting in a matrix dataset of dimensions $8 \times 7 \times 55$, which is denoted as Data-M. Note that the factor dimensions of Data-M remain unchanged from those of Data-V, i.e., $(q_c=2,q_r=1)$.
	
	To select the factor dimensions $(q_c, q_r)$, we apply BIC for both BFA and \tbfa on Data-M. The candidate model set is $\mathcal{I}=\{(q_c, q_r) \mid 1 \leq q_c \leq 4, 1 \leq q_r \leq 3\}$, including a total of 12 models. \reff{fig:sport.bic} (b) shows the evolutions of the BIC values of all candidates with varying $(q_c,q_r)$ on Data-M for both BFA and \emph{t}BFA. \reft{tab:sport.bic} summarizes the optimal results. It can be seen from \reff{fig:sport.bic} (b) and \reft{tab:sport.bic} that both matrix-based factor analysis models correctly select the true model. Notably, the BIC values for \tbfa are smaller than that of BFA model, indicating that \emph{t}BFA provides a better fit to the Data-M.
	
	\begin{table}[htb]
		\centering
		\caption{\label{tab:sport.bic} Results of the optimal models.} 
		\begin{tabular}{ccccc}
			\toprule
			Dataset                   & Method       & No. of factors    &  BIC       & $\hat{\nu}$\\ \midrule
			\multirow{2}{*}{Data-V}   & FA           & 2                 & 599.93     & — \\
			& \tfa         & 2                 & 541.76     & 4.28 \\
			\multirow{2}{*}{Data-M}   & BFA          & (2,1)             & -4554.03   & — \\
			& \tbfa        & (2,1)             & -4978.74   &7.23 \\ \bottomrule
		\end{tabular}
	\end{table}
	
	To further compare the results of FA ($q=2$) and \emph{t}FA ($q=2$) on Data-V with BFA ($q_c=2, q_r=1$) and \emph{t}BFA ($q_c=2, q_r=1$) on Data-M, varimax rotation \citep{lawley_fasm} is applied to the factor loading matrices $\hat{\bA}$ in FA and \emph{t}FA, as well as the column factor loading matrices $\hat{\bC}$ in BFA and \emph{t}BFA. The rotated factor loading results are presented in \reft{tab:sport.loading}. Despite some numerical differences, all four methods exhibit consistent factor structure patterns. Specifically, one factor captures explosive power for short-distance races and speed and interprets variables $\bx_1,\bx_2,\bx_3$, and another factor captures endurance for long-distance races and interprets the remaining variables, i.e., $\bx_4, \bx_5, \bx_6, \bx_7, \bx_8$. The experimental results indicate that the BIC-based \emph{t}BFA successfully extracts meaningful factors and performs well as expected on Data-M. Note that it is infeasible to apply \emph{t}FA to the vectorized Data-M as vectorization yields a dataset of dimensions $56 \times 55$, facing the high-dimensional problem of $d>N$. 
	\begin{table}[htb]
		\centering
		\caption{\label{tab:sport.loading} Comparison of factor loading matrices obtained by different methods.}
		\begin{tabular}{ccccccccc}
			\toprule
			Method              & $\bx_1$ & $\bx_2$ & $\bx_3$ & $\bx_4$ & $\bx_5$ & $\bx_6$ & $\bx_7$ & $\bx_8$  \\ \midrule
			\multirow{2}{*}{\emph{t}BFA}   &-1.20  &-1.54  &-2.19   &\cellcolor{blue!15}-2.72  &\cellcolor{blue!15}-3.10 &\cellcolor{blue!15}-3.44 &\cellcolor{blue!15}-3.45  &\cellcolor{blue!15}-3.52 \\
			&\cellcolor{blue!15}-3.54   &\cellcolor{blue!15}-3.47   &\cellcolor{blue!15}-2.94   & -2.51   & -2.14   & -1.62   & -1.59   & -1.15 \\ \cmidrule{2-9}
			\multirow{2}{*}{BFA}          & -1.56   & -1.94   & -2.67   &\cellcolor{blue!15}-3.27   &\cellcolor{blue!15}-3.72   &\cellcolor{blue!15}-4.10   &\cellcolor{blue!15}-4.10   &\cellcolor{blue!15}-4.13 \\
			&\cellcolor{blue!15}-4.40   &\cellcolor{blue!15}-4.29   &\cellcolor{blue!15}-3.62   & -3.08   & -2.62   & -1.98   & -1.95   & -1.42 \\ \cmidrule{2-9}
			\multirow{2}{*}{\emph{t}FA}     & 0.13    & 0.24    & 0.37    &\cellcolor{blue!15}0.44    &\cellcolor{blue!15}0.56    &\cellcolor{blue!15}0.62    &\cellcolor{blue!15}0.62    &\cellcolor{blue!15}0.62 \\
			&\cellcolor{blue!15}0.62    &\cellcolor{blue!15}0.70    &\cellcolor{blue!15}0.55    & 0.41    & 0.34    & 0.22    & 0.21    & 0.14 \\ \cmidrule{2-9}
			\multirow{2}{*}{FA}        & 0.28    & 0.38    & 0.54    &\cellcolor{blue!15}0.68    &\cellcolor{blue!15}0.79    &\cellcolor{blue!15}0.89    &\cellcolor{blue!15}0.90    &\cellcolor{blue!15}0.91 \\
			&\cellcolor{blue!15}-0.91   &\cellcolor{blue!15}-0.88   &\cellcolor{blue!15}-0.74   & -0.62   & -0.53   & -0.39   & -0.40   & -0.28 \\ \bottomrule
		\end{tabular}
	\end{table}
	
	\subsubsection{Fama–French 10 by 10 return series}
	In this experiment, we use the Fama–French 10 by 10 return series \footnote{Available from \url{http://mba.tuck.dartmouth.edu/pages/faculty/ken.french/data_library.html}.} \cite{wang2019factor}. This dataset comprises 624 matrix observations, covering monthly returns from January 1964 to December 2015, a total of 624 months. Each observation is a 10$\times$10 matrix, where the column variables represents the ten levels of market capital (Size) from small to large, and the row variables represents the ten levels of book-to-equity ratio (BE) from small to large. Based on this, stocks are classified into 100 portfolios. Following \cite{wang2019factor}, we preprocess the data by performing market adjustment on each observation, namely subtracting the corresponding monthly excess market returns to obtain market-adjusted return series.
	
	To assess the normality assumption for the data, we performe the Jarque-Bera test for each variable (i.e., portfolio). For convenience, the set of all 100 variables are denoted by $\Omega$. The results are reported in \ref{tab:test.return}. Based on the reported $p$-values, the normal assumption for 99 out of 100 variables are rejected, indicating the data does not follow the matrix-variate normal distribution.
	\begin{table}[htb]
		\centering
		\caption{\label{tab:test.return} Jarque-Bera test for the vectorized Fama-French 10 by 10 return series.}
		\begin{tabular}{ccc}
			\toprule
			$p$-value        & No. of variables    & Variable index \\ \midrule
			$p\leq0.001$\        & 94                     & $\Omega-\{10,25,43,46,53,74\}$ \\
			$0.001 < p\leq0.05$    & 5                      &  $\{25,43,46,53,74\}$ \\
			$p>0.05$            & 1                      & $\{10\}$   \\  \bottomrule
		\end{tabular}
	\end{table}
	
	
	We use BFA and \tbfa to analyze this dataset. BIC is used to select the factor dimensions $(q_c, q_r)$ for both models. The candidate set is $\mathcal{I}=\{(q_c,q_r) \mid 1 \leq q_c \leq 6, 1 \leq q_r \leq 6\}$. \reff{fig:return.bic} shows the evolutions of BIC values for BFA and \tbfa as $(q_c,q_r)$ takes different values. \reft{tab:return.bic} collects the optimal results. From \reff{fig:return.bic} and \reft{tab:return.bic}, it can be observed that both BFA and \tbfa achieve the minimum BIC values when the factor dimensions are $(q_c=2, q_r=3)$. Moreover, it is observed that the BIC values for all candidate models of \tbfa are lower than the corresponding BFA models, indicating that \tbfa provides a better fit compared to BFA.
	
	
	\begin{figure*}[htb]
		\centering \scalebox{0.8}[0.6]{\includegraphics*{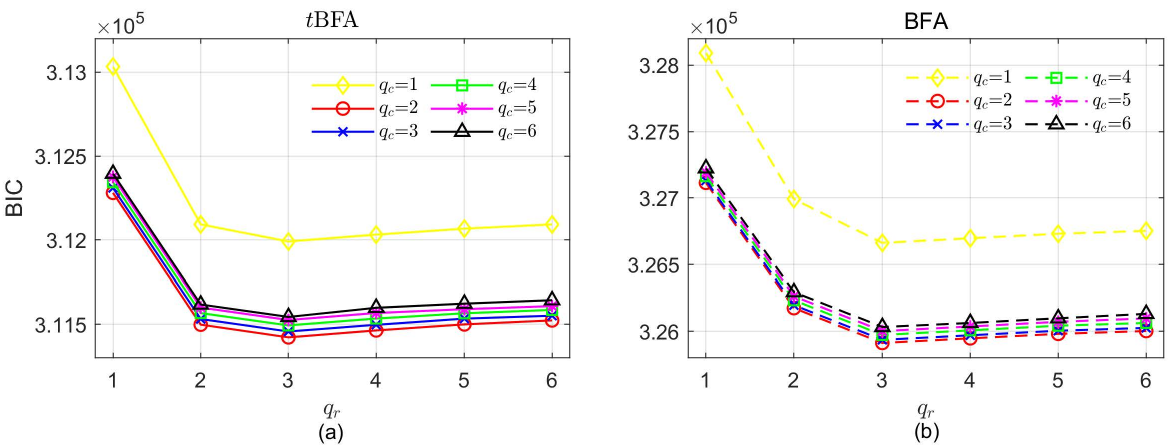}} 
		\caption{BIC values versus various $(q_c,q_r)$ for (a) \tbfa and (b) BFA models.}
		\label{fig:return.bic}
	\end{figure*}
	
	\begin{table}[htb]
		\centering
		\caption{\label{tab:return.bic} The results of optimal models for \tbfa and BFA.} 
		\begin{tabular}{cccc}
			\toprule
			Method       & Choice of $(q_c,q_r)$    &  BIC         & $\hat{\nu}$\\ \midrule
			\tbfa        & (2,3)                    & 311422.78    &  6.29 \\
			BFA          & (2,3)                    & 325912.98    &    —    \\  \bottomrule
		\end{tabular}
	\end{table}
	
	To further examine the obtained column (Size) and row (BE) common factors, the varimax rotation is applied to the column and row factor loading matrices, respectively. \reft{tab:return.load} shows the results scaled by 30. For comparison, the result by MFM taken from \cite{wang2019factor} is also included. From \reft{tab:return.load}, it can be observed that
	\begin{enumerate}[(i)]
		\item Size factors. The two common factors obtained by \tbfa and BFA are comparable. Specifically, the first factor could be named `small size factor', mainly explaining the portfolios with the smallest market capital (from S1 to S5). The second factor could be named `big size factor', primarily explaining portfolios with larger market capital (from S6 to S10). In constrast, there are possibly two or three factors for MFM as the largest S10 portfolio may be classified as a single factor.  
		
		\item BE factors. Both \tbfa and BFA obtain three common factors, while MFM gives two factors. In detail, for \tbfa, the first factor could be named `medium BE factor', as it mainly explains portfolios with BE levels from 5 to 7, i.e., medium book-to-equity ratios. Likewise, the second factor could be named `high BE factor', mainly explaining portfolios with BE levels from 8 to 10, namely relatively high book-to-equity ratios), and the third factor named `small BE factor', mainly explaining portfolios with BE levels from 1 to 4, namely relatively low book-to-equity ratios. We note that \cite{fama1995size} defined the lowest 30\% of BE as Low, the highest 30\% as High, and the middle 40\% as Medium. This is roughly consistent with the factors obtained by \tbfa and BFA, while the factors by MFM are not the case.
	\end{enumerate}
	
	\begin{table}[htb]
		\centering
		\caption{\label{tab:return.load} Loading matrices by \tbfa, BFA and MFM on Fama–French data set, }
		\begin{tabular}{cccccccccccc}
			\toprule
			\multicolumn{12}{c}{Size} \\ \midrule
			Method              &Factor & S1      & S2      & S3     & S4     & S5     & S6    & S7    & S8    & S9    & S10  \\ \midrule
			\multirow{2}{*}{\tbfa}   &1      &\cellcolor{blue!15}-26  &\cellcolor{blue!15}-28 &\cellcolor{blue!15}-23 &\cellcolor{blue!15}-21  &\cellcolor{blue!15}-18      &-14      &-10      &-8      &-4      &3 \\
			&2      &1        &-1       &-9      &-10     &-12     &\cellcolor{blue!15}-15 &\cellcolor{blue!15}-15 &\cellcolor{blue!15}-17  &\cellcolor{blue!15}-15 &\cellcolor{blue!15}-11 \\ \cmidrule{2-12}
			\multirow{2}{*}{BFA}        &1    &\cellcolor{blue!15}-29  &\cellcolor{blue!15}-29 &\cellcolor{blue!15}-24 &\cellcolor{blue!15}-20 &\cellcolor{blue!15}-17   &-14      &-10      &-8      &-4      &4  \\
			&2     &-1       &2        &10      &11      &14   &\cellcolor{blue!15}17 &\cellcolor{blue!15}17 &\cellcolor{blue!15}18 &\cellcolor{blue!15}16     &\cellcolor{blue!15}9 \\ \cmidrule{2-12}
			\multirow{2}{*}{MFM}     &1   &\cellcolor{blue!15}-13  &\cellcolor{blue!15}-14 &\cellcolor{blue!15}-13 &\cellcolor{blue!15}-13 &\cellcolor{blue!15}-10   &-5      &-2      &1      &6   &\cellcolor{blue!15}7  \\
			&2    &0       &0        &-2      &3      &5   &\cellcolor{blue!15}12 &\cellcolor{blue!15}12 &\cellcolor{blue!15}18 &\cellcolor{blue!15}15    &\cellcolor{blue!15}5 \\ \midrule 
			\multicolumn{12}{c}{Book-to-Equity} \\ \midrule
			Method              &Factor   & BE1     & BE2     & BE3    & BE4    & BE5    & BE6     & BE7    & BE8    & BE9    & BE10  \\ \midrule
			\multirow{3}{*}{\tbfa}   &1        & 7   & 11  & 18  & 20   &\cellcolor{blue!15}23  &\cellcolor{blue!15}23 &\cellcolor{blue!15}24 &16    & 11    & 9\\
			&2    &-5    &-7    &-8   &-10   &-10   &-13   &-15   &\cellcolor{blue!15}-31 &\cellcolor{blue!15}-36 &\cellcolor{blue!15}-43 \\ 
			&3   &\cellcolor{blue!15}-39 &\cellcolor{blue!15}-35 &\cellcolor{blue!15}-25 &\cellcolor{blue!15}-21  &-16  &-13    &-8    &-5   &-7   &-9 \\ \cmidrule{2-12}
			\multirow{3}{*}{BFA}        &1   &9  &17 &\cellcolor{blue!15}29 &\cellcolor{blue!15}29 &\cellcolor{blue!15}32 &\cellcolor{blue!15}31 &\cellcolor{blue!15}35 &20 &19 & 10\\
			&2   &-6    &-7    &-8   &-13   &-12   &-15    &-15   &\cellcolor{blue!15}-47 &\cellcolor{blue!15}-43 &\cellcolor{blue!15}-61 \\
			&3   &\cellcolor{blue!15}48  &\cellcolor{blue!15}41  &29  &24  &17  &12    &7    &6   &8   &10 \\  \cmidrule{2-12}
			\multirow{2}{*}{MFM}   &1 &\cellcolor{blue!15}-21 &\cellcolor{blue!15}-14 &\cellcolor{blue!15}-11 &\cellcolor{blue!15}-9 &-4  &-1   &-1   &-4  &1  &3 \\
			&2  &\cellcolor{blue!15}-9  &2  &3  &7  &\cellcolor{blue!15}9 &\cellcolor{blue!15}10 &\cellcolor{blue!15}10 &\cellcolor{blue!15}10 &\cellcolor{blue!15}13    &\cellcolor{blue!15}14 \\ \bottomrule   
		\end{tabular}
	\end{table}
	
	\reff{fig:return.time} plots the estimated factor matrix over time. From \reff{fig:return.time}, it can be seen that, given the market size factor, the volatility of the `low BE factor' is significantly lower than that of the `medium and high BE factors'. This is because low BE indicates the typical performance of companies with higher average capital return rates (i.e., growth stocks), while high BE ratios often signify sustained low returns on book value, indicative of the typical performance of relatively poor-performing companies \citep{fama1995size}. Furthermore, when the BE factor is given, the larger the market capital, the more stable the return.
	\begin{figure*}[htb]
		\centering \scalebox{0.7}[0.5]{\includegraphics*{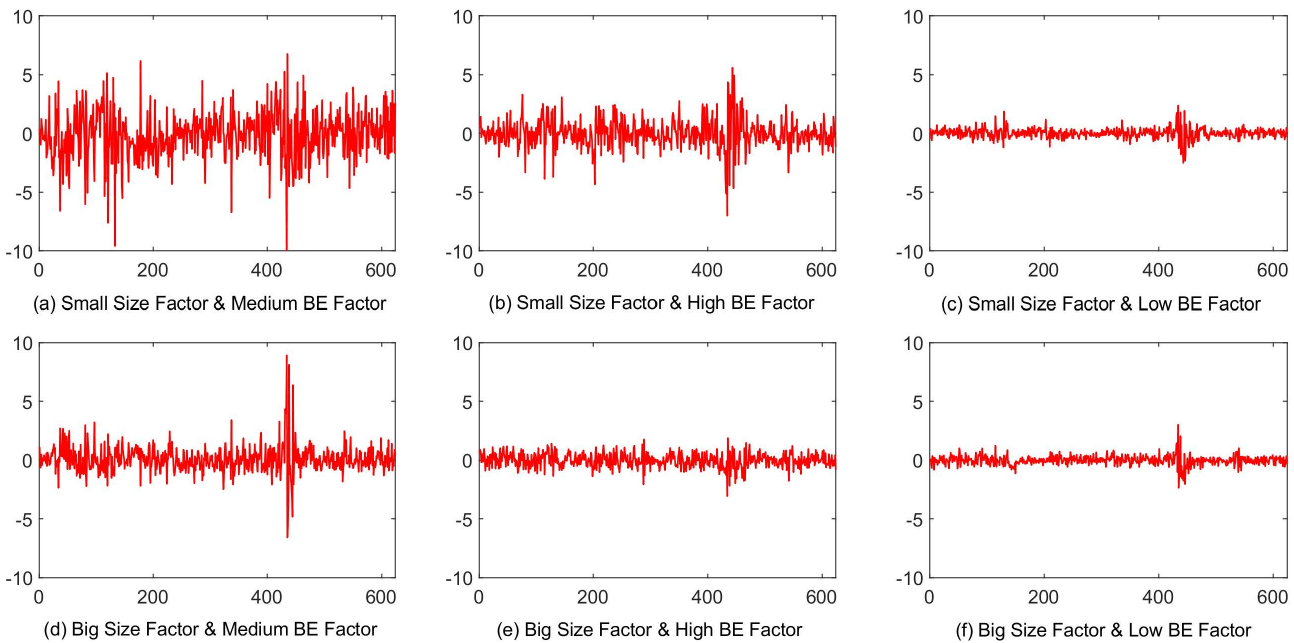}} 
		\caption{Fama–French series: Estimated factors over time by \tbfa.}
		\label{fig:return.time}
	\end{figure*}

	\section{Conclusion and future works}\label{sec:conclusion} 
	To improve BFA in the presence of heavy-tailed matrix data or matrix-valued outliers, we propose in this paper a novel factor analysis tool called \tbfa, using the matrix-variate \emph{t} distribution. For the task of parameter estimation, we present two EM-type algorithms, namely, ECME and AECM. AECM scales better on high-dimensional matrix data while ECME converges faster. To expedite both ECME and AECM, we incorporate the PX technique and introduce two enhanced variants, namely PX-ECME and PX-AECM, demonstrating clear advantages over their non-PX counterparts. To assess the accuracy of the obtained ML estimator, we explicitly drive the closed-form expression of Fisher information matrix. Unlike \tfa, where factor scores are vectors, the factor scores in \tbfa are matrices.
	We also conduct a visual analysis of the matrix factor scores. Our empirical results on matrix datasets show that \tbfa improves the robustness of BFA as expected. More importantly, \tbfa performs much better than its vector-based cousin \tfa, primarily owing to its higher breakdown point. 
	
	It is important to highlight that the newly proposed \tbfa is tailored for matrix-valued data. In contemporary applications, there is a growing prevalence of naturally occurring tensor-valued data. For instance, the mortality data analyzed in \citep{fosdick2014-sfa} consists of four modes: mode 1 encompasses 40 countries, mode 2 involves 9 time points, mode 3 encompasses 23 age groups, and mode 4 includes 2 gender groups. We consider extending the proposed \tbfa model to accommodate tensor-valued data as a future research direction. 
	
	For future works, inspired by the works on the vector-based mixtures of factor analysis \citep{zhao2008-mfa-ecm,zhao2014-hbic-mppca,wang2022mtfa}, it would be interesting to extend our \tbfa in this paper to mixtures of \tbfa for robustly clustering matrix data.
	
	\section*{Acknowledgements} 
	This work was supported by the National Natural Science Foundation of China under Grant 12161089.
	
	\appendix	
%
%
%

		\section{Generating the three types of outliers}\label{sec:gout}
		Let $\bC_{fc}$ (resp. $\bR_{fc}$) be the $d_c \times q_c$ (resp. $d_r \times q_r$) factor loading matrix, spanning the column (resp. row) factor component subspace. Correspondingly, $\bC_{oc}$ and $\bR_{oc}$ denote their orthogonal complements. To introduce outliers, inspired by \cite{she2016robust}, we utilize the following contaminated model, which is a modification of \tbfa \refe{eqn:tbfa}:
		
		\begin{equation}\label{eqn:outliers}
			\bX_o = \bW + \bC_{fc}\bZ\bR'_{fc}+\bC_{fc}\beps_r+\beps_c\bR_{fc}'+\beps + \bC_o\bZ_o\bR'_o.
		\end{equation}
		Depending on the type of outliers, $\bC_o, \bZ_o$, and $\bR_o$ in \refe{eqn:outliers} are set as follows. 
		
		(i) FC outliers, $\bC_o = [\bC_{fc},\bO]$ and $\bR_o = [\bR_{fc},\bO]$, where $\bO$ is a zero matrix of suitable dimension. $\bZ_o$ is a zero matrix except that the entries of the submatrix comprising the first $q_c$ columns and $q_r$ rows are drawn from a uniform distribution.
		
		(ii) OC outliers, $\bC_o = [\bO,\bC_{oc}]$ and $\bR_o = [\bO,\bR_{oc}]$. $\bZ_o$ is a zero matrix except that the entries of the submatrix comprising the last $d_c-q_c$ columns and $d_r-q_r$ rows are drawn from a uniform distribution.
		
		(iii) FC+OC outliers, $\bC_o = [\bC_{fc},\bC_{oc}]$ and $\bR_o = [\bR_{fc},\bR_{oc}]$, the entries of $\bZ_o$ are drawn from a uniform distribution.
		
		The uniform distribution in (i), (ii), and (iii) is one of those in the four situations described in \refs{sec:rbst}.

	\bibliography{journals,tbfa,jhzhao-pub}
	\bibliographystyle{elsarticle-harv}
	
\end{document}